\documentclass[review,5p]{elsarticle}

\journal{Neurocomputing}

\usepackage{multirow}
\usepackage{amsmath}
\usepackage{amsthm}
\usepackage{amssymb}
\usepackage{mathtools}
\usepackage{booktabs}
\usepackage{xcolor}
\usepackage{nicefrac} 
\usepackage{multirow}
\usepackage{tikz}
\usetikzlibrary{arrows}
\usepackage{enumitem}
\usepackage{algorithm}
\usepackage{algorithmic}
\usepackage{bm}
\usepackage{graphicx}
\usepackage{subcaption}
\usepackage{caption}

\newcommand{\vect}[1]{\mathbf{#1}}

\newtheorem{thm}{Theorem}[]
\newtheorem{thmA}{Theorem}[]
\newtheorem{lem}{Lemma}[]
\newtheorem{cor}{Corollary}[]
\newtheorem{prop}{Proposition}[]
\newtheorem{propA}{Proposition}[]

\DeclareMathOperator*{\argmax}{arg\,max}
\DeclareMathOperator*{\argmin}{arg\,min}

\usepackage{soul}
\DeclareRobustCommand{\cc}[1]{{#1}}
\soulregister{\cc}{1}

\begin{document}

\begin{frontmatter}

\title{ASAT: Adaptively Scaled Adversarial Training in Time Series}

\author[address1]{Zhiyuan Zhang}
\ead{zzy1210@pku.edu.cn}
\author[address1]{Wei Li}
\ead{liweitj47@pku.edu.cn}
\author[address2]{Ruihan Bao}
\ead{ruihan.bao@mizuho-sc.com}
\author[address2]{Keiko Harimoto}
\ead{keiko.harimoto@mizuho-sc.com}
\author[address1]{Yunfang Wu}
\ead{wuyf@pku.edu.cn}
\author[address1]{Xu Sun\corref{mycorrespondingauthor}}
\ead{xusun@pku.edu.cn}

\address[address1]{MOE Key Laboratory of Computational Linguistics, School of Computer Science, Peking University, Beijing, China.}
\address[address2]{Mizuho Securities Co., Ltd, Japan.}

\cortext[mycorrespondingauthor]{Corresponding author.}

\begin{abstract}
Adversarial training is a method for enhancing neural networks to improve the robustness against adversarial examples. Besides the security concerns of potential adversarial examples, adversarial training can also improve the generalization ability of neural networks, train robust neural networks, and provide interpretability for neural networks. In this work, we introduce adversarial training in time series analysis to enhance the neural networks for better generalization ability by taking the finance field as an example. Rethinking existing research on adversarial training, we propose the adaptively scaled adversarial training (ASAT) in time series analysis, by rescaling data at different time slots with adaptive scales. Experimental results show that the proposed ASAT can improve both the generalization ability and the adversarial robustness of neural networks compared to the baselines. Compared to the traditional adversarial training algorithm, ASAT can achieve better generalization ability and similar adversarial robustness.
\end{abstract}
\end{frontmatter}

\section{Introduction}

Neural networks are found vulnerable to adversarial examples~\citep{Intriguing_properties_of_neural_networks,Explaining_and_Harnessing_Adversarial_Examples,Adversarial_examples_in_the_physical_world,Deepfool} despite their promising performance. Adversarial examples are generated by adding small malicious perturbations to the input data that can mislead models, which reveal the vulnerability of neural networks with respect to the input data. Adversarial training~\citep{Explaining_and_Harnessing_Adversarial_Examples,Towards_Evaluating_the_Robustness_of_Neural_Networks,YOPO,ForFree,freeLB} is a method for enhancing neural networks that can improve both the adversarial robustness and the generalization ability at the same time. Besides the security concerns that neural network faces potential risks of adversarial example attack, researches on adversarial training focus on improving the accuracy of the neural networks~\citep{YOPO,ForFree,freeLB}, training robust neural networks~\citep{robust_domain_text,Robsut_Translation_Doubly_Adversarial,Robust_Machine_Translation}, and providing interpretability for neural networks~\citep{Intriguing_properties_of_neural_networks,UAP,Decision_Boundary}.

Existing researches on adversarial training to enhance the neural networks mainly focus on the CV~\citep{Intriguing_properties_of_neural_networks,Explaining_and_Harnessing_Adversarial_Examples,YOPO,ForFree} and NLP~\citep{robust_domain_text,Robsut_Translation_Doubly_Adversarial,Robust_Machine_Translation,freeLB} fields. In the finance field, some researches~\citep{GAN19,GAN21,simonetto2018generatingspikingGAN,zhang2018GANingrids,ramponi2018tcGAN,yoon2019timeGAN,takahashi2019modelingwithGAN,data-Augmentation-GANs,Clinical-time-series-GAN,TTS-GAN} about generative adversarial networks (GANs)~\citep{DBLP:journals/corr/GAN} are proposed to generate or synthesize better time-series data that can preserve temporal dynamic in time-series. However, in this work, we introduce adversarial training in time series analysis or the finance field to enhance the neural networks for the generalization purpose.

Traditional adversarial training methods adopt the $L_p$-norm constraint that treats every dimension of perturbations symmetrically. Rethinking the shift-invariant hypothesis~\citep{Convolutional_Networks_Shift-Invariant_Again,Truly_shift-invariant_convolutional_neural_networks} of convolution neural networks in the computer vision (CV) field and the equivalence of different dimensions of word embeddings in the natural language processing (NLP) field~\citep{hotflip,freeLB}, it is roughly reasonable that every dimension of the input data is treated as of similar significance by the $L_p$-norm constraint in adversarial training methods. However, in time series, different dimensions of inputs are not symmetrical. Therefore, we propose to adopt a constraint with time-dependent or adaptive scales in different dimensions. In our work, we propose adaptively scaled adversarial training (ASAT) in time series analysis. The main purpose of our work is to improve the generalization ability of neural networks.

To validate the effects of ASAT, we implement some representative baselines on a representative time series task on the volume prediction task in the finance field, including moving baselines and several neural network baselines. Experimental results show that ASAT can improve the generalization ability of neural networks compared to both the baselines and the traditional adversarial training algorithm. Moreover, ASAT can improve the adversarial robustness of neural networks compared to the baselines and achieve similar adversarial robustness compared to the traditional adversarial training algorithm.

Besides enhancing neural networks, we also try to explain the decision bases of black-box neural networks towards more interpretability for neural network decisions via the proposed dimension-wise adversarial sensitivity indicator. 
With the proposed indicator, we have some interesting findings: \textit{(1) ASAT can alleviate overfitting.} As analyzed in Sec.~\ref{sec:overfit}, baseline models with ASAT are sensitive to different dimensions during multiple training, which indicates that baseline models tend to overfit to some false clues, while ASAT can alleviate it. 
\textit{(2) ASAT can help models capture fluctuations in time series.} As analyzed in Sec.~\ref{sec:single}, we make detailed examinations of a single data instance and find that ASAT can help models to be more sensitive to abnormal fluctuations. The decision bases of models tend to be more reasonable with ASAT. \textit{(3) Models tend to pay close attention to the volumes of recent time slots.} Models tend to be sensitive to volumes of recent time slots. Since recent time slots are more important in time series regularization, it is reasonable and also accords with human intuition.

Our contributions are summarized as follows:
\begin{itemize}
\item Rethinking existing adversarial training approaches, we propose to rescale dimensions of the perturbation according to their importance and adopt time-dependent or adaptive scales with multi-step risk averaging. We propose an adaptively scaled adversarial training (ASAT) algorithm.
\item We adopt ASAT on various baselines, including our newly proposed Transformer baselines. Results show that ASAT can improve both the generalization ability and the adversarial robustness of neural networks compared to the baselines. Compared to the traditional adversarial training algorithm, ASAT can achieve better generalization ability and similar adversarial robustness. 
\item With the proposed dimension-wise adversarial sensitivity indicator, we probe the sensitivities of different input dimensions and explain the decision bases of black-box neural networks, which provides more interpretability for neural network decisions.
\end{itemize}

\section{Adversarial Training}
In this section, we first introduce the concept of adversarial examples and adversarial training. Then, we describe the proposed dimension-wise adversarial sensitivity indicator inspired by the adversarial sensitivity.

\subsection{Adversarial Examples}

Let $\mathcal{D}$ denote the dataset, $\vect{x}=(x_1, x_2,\cdots, x_k)^\text{T}\in\mathbb{R}^k$ and $y$ stand for a data input and its ground truth, $f$ denote the neural network with the parameter vector $\bm \theta$ and $\mathcal{L}$ denote the loss function. The \textbf{adversarial example}~\citep{Intriguing_properties_of_neural_networks,Explaining_and_Harnessing_Adversarial_Examples} can be defined as the small perturbation $\bm\delta=(\delta_1, \delta_2,\cdots, \delta_k)^\text{T}\in\mathbb{R}^k$ on the data input $\vect{x}$ that can mislead the model and cause the maximum increases in the loss function:
\begin{align}
\bm\delta = \argmax_{\bm\delta\in S}\mathcal{L}(f(\vect{x}+\bm\delta,\bm\theta), y),
\label{eq:adversarial_example}
\end{align}
where $S$ denotes the \textbf{constraint set}, the perturbation is specified by $S$. For example, $S=\{\bm\delta:\|\bm\delta\|_2\le \epsilon\}$ specifies that the perturbation should be in a $k$-dimensional hyperball with a radius of $\epsilon$. Typical constraints are $L_p$-norm bounded ($1\le p\le+\infty$) and $S=\{\bm\delta:\|\bm\delta\|_p\le \epsilon\}$, where the $L_p$-norm is defined as:
\begin{align}
\|\bm\delta\|_p=\big(\sum\limits_{i=0}^{k}|\delta_i|^p\big)^{\frac{1}{p}},\quad \|\bm\delta\|_{+\infty}=\max\limits_{1\le i\le k}|\delta_i|.
\end{align}

Since neural networks are sensitive to malicious perturbations in inputs, \textit{i.e.}, \textbf{adversarial attacks}, the \textbf{adversarial risk} is defined as the loss after adversarial attacks and the \textbf{empirical risk} is defined as the initial loss:
\begin{align}
\text{Risk}_\text{adv}(S)&:=\mathbb{E}\big[\max\limits_{\bm\delta\in S}\mathcal{L}(f(\vect{x}+\bm\delta, \bm\theta), y)],\\
\text{Risk}_\text{emp}&:=\mathbb{E}\big[\mathcal{L}(f(\vect{x}, \bm\theta), y)\big].
\end{align}

The loss increases caused by adversarial attacks can be defined as the \textbf{adversarial sensitivity}: 
\begin{align}
\mathcal{R}_\text{adv}(S):&=\text{Risk}_\text{adv}(S)-\text{Risk}_\text{emp}\\
&=\mathbb{E}\big[\max\limits_{\bm\delta\in S}\mathcal{L}(f(\vect{x}+\bm\delta, \bm\theta), y)-\mathcal{L}(f(\vect{x}, \bm\theta), y)\big].
\end{align}

The adversarial risk or sensitivity can be defined on a dataset or a single data instance. The adversarial risk or sensitivity on the test dataset can evaluate the robustness of neural networks with respect to adversarial attacks, a lower adversarial sensitivity or risk means a better \textbf{adversarial robustness}.

\subsection{Generating Adversarial Attacks}
To generate adversarial attacks, the fast gradient methods~\citep{Explaining_and_Harnessing_Adversarial_Examples}, including the fast gradient sign method (FGSM) and the fast gradient method (FGM), propose to optimize Eq.(\ref{eq:adversarial_example}) via optimizing the first-order Taylor Expansion term of the loss function instead when $S=\{\bm\delta:\|\bm\delta\|_p\le \epsilon\}$, namely maximizing the inner product of $\bm\delta$ and the gradient $\vect{g}=\nabla_{\vect{x}}\mathcal{L}( f(\vect{x},\bm\theta), y)$:
\begin{align}
\bm\delta = \argmax_{\bm\delta\in S}\bm\delta^{\text{T}}\vect{g}=\epsilon\text{sgn}(\vect{g})\odot\frac{|\vect{g}|^\frac{1}{p-1}}{\||\vect{g}|^\frac{1}{p-1}\|_p},
\label{eq:fast_gradient}
\end{align}
where $\odot$ denotes element-wise product and sgn denotes element-wise sign function. FGSM adopts $p=+\infty$ and FGM adopts $p=2$.

Another representative line of existing studies for generating adversarial examples is PGD-based attacks~\citep{Towards_Deep_Learning_Models_Resistant_to_Adversarial_Attacks,Unified-min-max}. Suppose $S=\{\bm\delta:\|\bm\delta\|_p\le \epsilon\}$. The projected gradient descent (PGD) algorithm optimizes Eq.(\ref{eq:adversarial_example}) via a multi-step variant of fast gradient methods~\citep{Explaining_and_Harnessing_Adversarial_Examples}:
\begin{align}
\bm\delta_{k+1} = \Pi_S(\bm\delta_k + \vect{u}_{k+1}),
\label{eq:PGD} 
\end{align}
where $\bm\delta_0$ is usually set to the zero vector and $\vect{u}_k$ is the $k$-th update in PGD. It is usually solved in Eq.(\ref{eq:fast_gradient}) and defined as $\vect{u}_k=\argmax_{\vect{u}\in S_1}\vect{u}^{\text{T}}\vect{g}_{k-1}$, where $\vect{g}_{k-1}$ is defined as the gradient of perturbed input $\vect{x}+\bm\delta_{k-1}$, $\vect{g}_{k-1}=\nabla_{\vect{x} + \bm\delta_{k-1}}\mathcal{L}( f(\vect{x} + \bm\delta_{k-1},\bm\theta), y)$, where the constraint is $S_1=\{\vect{u}:\|\vect{u}\|_p\le \tau\}$ and the step size is set to $\tau$. $\Pi_S(\vect{v})$ is the projection function that projects $\vect{v}$ into the set $S$. In two common cases, $L_2$-norm~\citep{freeLB} and $L_{+\infty}$-norm~\citep{Towards_Deep_Learning_Models_Resistant_to_Adversarial_Attacks}, the projection functions are:
\begin{align}
\Pi_{\{\vect{v}:\|\vect{v}\|_2\le\epsilon\}}{(\vect{v})} &= \min\{\|\vect{v}\|_2,\epsilon\}\frac{\vect{v}}{\|\vect{v}\|_2},
\label{eq:projection_L2}\\ \Pi_{\{\vect{v}:\|\vect{v}\|_{+\infty}\le\epsilon\}}{(\vect{v})}  &= \text{clip}(\vect{v},-\epsilon, \epsilon),
\label{eq:projection_Linf}
\end{align}
where the $\text{clip}(\vect{h})$ function clips each dimension $h_i$ of $\vect{h}$ into $[-\epsilon, \epsilon]$, namely $\text{clip}(\vect{h})_i=\min(\max(h_i, -\epsilon), \epsilon)$.

\subsection{Traditional Adversarial Training}

Since neural networks are sensitive to adversarial examples, \textbf{adversarial training}~\citep{Explaining_and_Harnessing_Adversarial_Examples,Towards_Evaluating_the_Robustness_of_Neural_Networks} algorithms are designed to improve both the adversarial robustness and the generalization ability of neural networks. The target of traditional adversarial training is seeking to find the parameter vector $\bm \theta$ with lower adversarial risks:
\begin{align}
\bm\theta = \argmin_{\bm\theta}\mathbb{E}_{(\vect{x},y)\sim\mathcal{D}}\big[\max_{\bm\delta\in S}\mathcal{L}(f(\vect{x}+\bm\delta,\bm\theta), y)\big].
\label{eq:PGD_target}
\end{align}

\cite{Explaining_and_Harnessing_Adversarial_Examples} estimate the risk for adversarial attacks via the fast gradient methods. Other existing studies also estimate the risk via PGD-based attacks~\citep{Towards_Evaluating_the_Robustness_of_Neural_Networks,YOPO,ForFree,freeLB,Unified-min-max}. Fast gradient methods, including FGSM and FGM, can be seen as PGD-based attacks with the step number $K=1$. Besides, some PGD-based adversarial training algorithms~\citep{ForFree,freeLB} minimize the adversarial sensitivities of multiple perturbations generated in different steps.

\subsection{Dimension-wise Adversarial Sensitivity}

Inspired by the adversarial sensitivity, we propose to probe the sensitivities of different input dimensions. Besides the adversarial sensitivity of the whole model, we can also define the dimension-wise adversarial sensitivity indicator, namely the sensitivity of the $i$-th input dimension. Suppose $\mathcal{R}^{(i)}_\text{adv}(\epsilon)$ denotes the sensitivities of the $i$-th dimension, namely the expectation of the maximum loss change when the $i$-th dimension of the input changes no more than $\epsilon$, which is defined as:
\begin{align}
\mathcal{R}^{(i)}_\text{adv}(\epsilon):=\mathbb{E}\big[\max\limits_{|\delta_i|\le\epsilon, \delta_{j}=0 (i\ne j)}\mathcal{L}(f(\vect{x}+\bm\delta, \bm\theta), y)-\mathcal{L}(f(\vect{x}, \bm\theta), y)\big].
\label{eq:robustness}
\end{align}

The sensitivities of input dimensions reveal how the perturbations on different dimensions can cause the loss to increase. Neural networks pay more attention to dimensions with higher sensitivities and these dimensions tend to be relatively important dimensions for models to make decisions. Therefore, dimension-wise sensitivities can help reveal the decision basis of the model $f$ and the importance of different dimensions in the prediction process. In traditional factor analysis methods, the factor loading~\citep{factor_loading} measures the importance of a factor, which can be seen as the weight of the factor in a linear model roughly. The proposed sensitivity does not only cover traditional measurements in the linear model, as illustrated in Proposition~\ref{prop:linear}, but also evaluates the importance of different dimensions in the nonlinear neural networks, as illustrated in Proposition~\ref{prop:risk_non_linear}. We put the proofs in Appendix.

\begin{prop}
Consider a linear model $f(\vect{x}, \bm\theta)=\bm\theta^\text{T}\vect{x}+b$. Suppose the loss function $\mathcal{L}(\hat y,y)$ is a function $\ell(\hat y)$ with respect to $\hat y=f(\vect{x}, \bm\theta)$ and $M=\sup|\ell''|$ exists. Then there exists $C>0$, such that:
\begin{align}
\mathcal{R}^{(i)}_\text{adv}(\epsilon)=C\epsilon|\theta_i|+ O\big(\frac{M\theta_i^2\epsilon^2}{2}\big),\quad (i=1,2,\cdots,k),
\end{align}
namely $\mathcal{R}^{(i)}_\text{adv}(\epsilon)\propto|\theta_i|$ holds approximately when $\epsilon$ is small.
\label{prop:linear}
\end{prop}

Proposition~\ref{prop:linear} illustrates that the proposed sensitivity is proportional to the absolute value of the weight or the factor loading in a linear model. In our work, we adopt the MSE loss, $\mathcal{L}(\hat y,y)=(\hat y-y)^2$ and $M=2$.

\begin{prop}
Consider a neural network that its loss function is convex and $M$-smooth in the neighborhood of $\bm\theta$ \footnote{Note that $\mathcal{L}$ is only required to be convex and $M$-smooth in the neighborhood of $\bm\theta$ instead of the entire $\mathbb{R}^k$.}. Then, 
\begin{align}
\mathcal{R}^{(i)}_\text{adv}(\epsilon)=\epsilon\big(\mathbb{E}| g_i|\big)+ O\big(\frac{M\epsilon^2}{2}\big),\quad (i=1,2,\cdots,k),
\end{align}
where ${\vect{g}}=\nabla_{\vect{x}}\mathcal{L}( f(\vect{x},\bm\theta), y)$ and its $i$-th dimension is $g_i$. Namely $\mathcal{R}^{(i)}_\text{adv}(\epsilon)\propto\mathbb{E}|g_i|$ holds approximately when $\epsilon$ is small.
\label{prop:risk_non_linear}
\end{prop}

Proposition~\ref{prop:risk_non_linear} illustrates that the proposed dimension-wise sensitivities in neural networks are approximately proportional to the gradients.

\begin{figure*}[!t]
\centering
\subcaptionbox{$\|\bm\delta\|_2\le\epsilon$\label{fig:norm1}, circle.}{\includegraphics[height=1.7 in,width=0.24\linewidth]{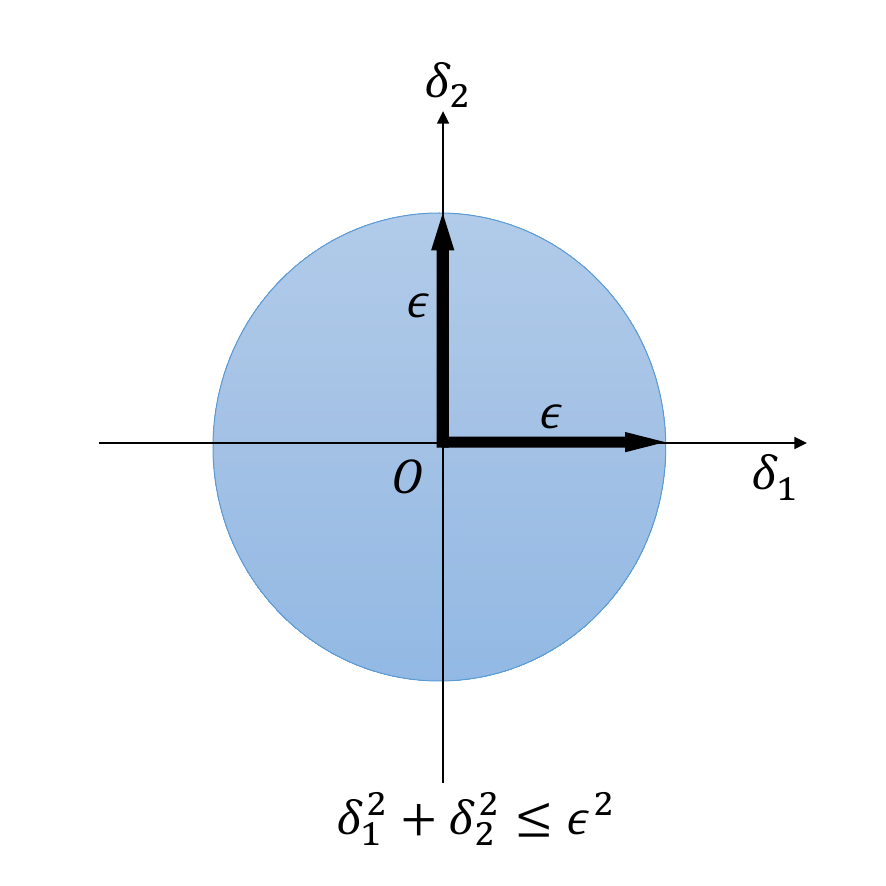}
}
\hfil
\subcaptionbox{$\|\bm\alpha^{-1}\odot\bm\delta\|_2\le\epsilon$, ellipse.\label{fig:norm2}}{\includegraphics[height=1.7 in,width=0.24\linewidth]{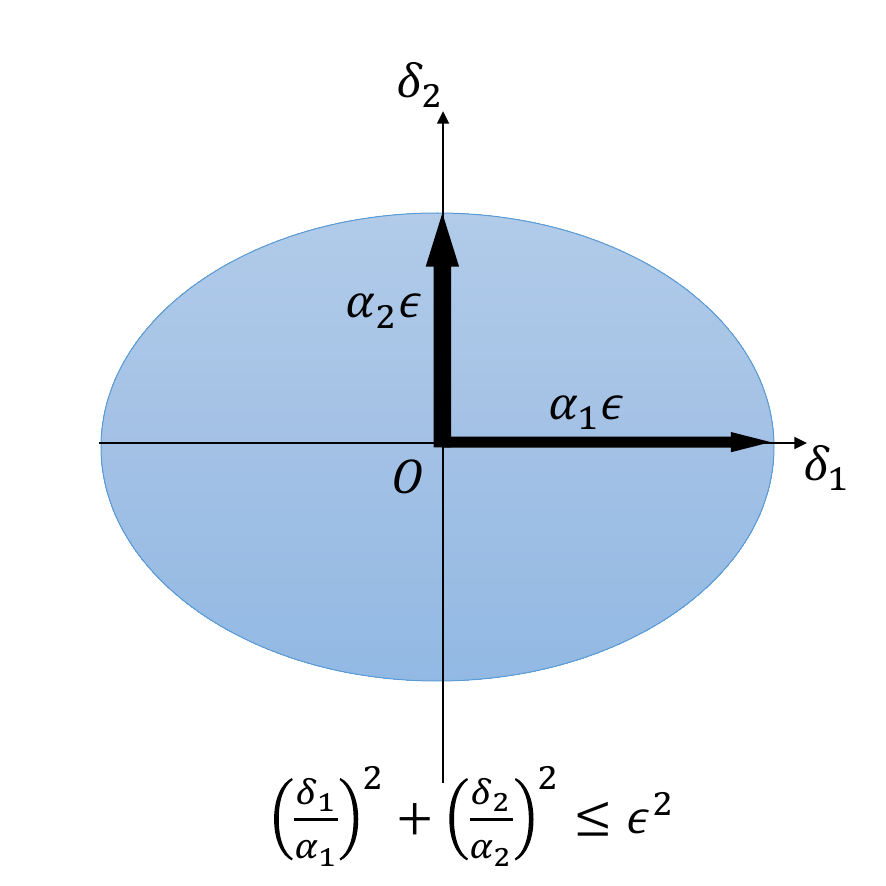}}
\hfil
\subcaptionbox{$\|\bm\delta\|_{+\infty}\le\epsilon$, square.\label{fig:norm3}}{\includegraphics[height=1.7 in,width=0.24\linewidth]{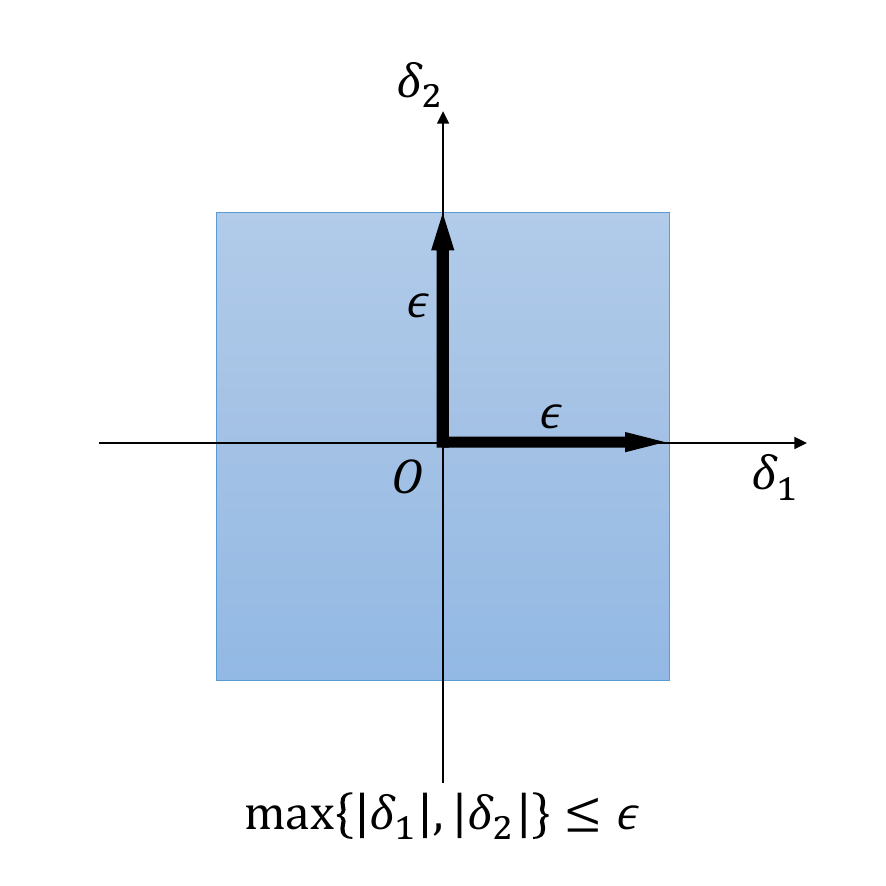}}
\hfil
\subcaptionbox{$\|\bm\alpha^{-1}\odot\bm\delta\|_{+\infty}\le\epsilon$, rectangle.\label{fig:norm4}}{\includegraphics[height=1.7 in,width=0.24\linewidth]{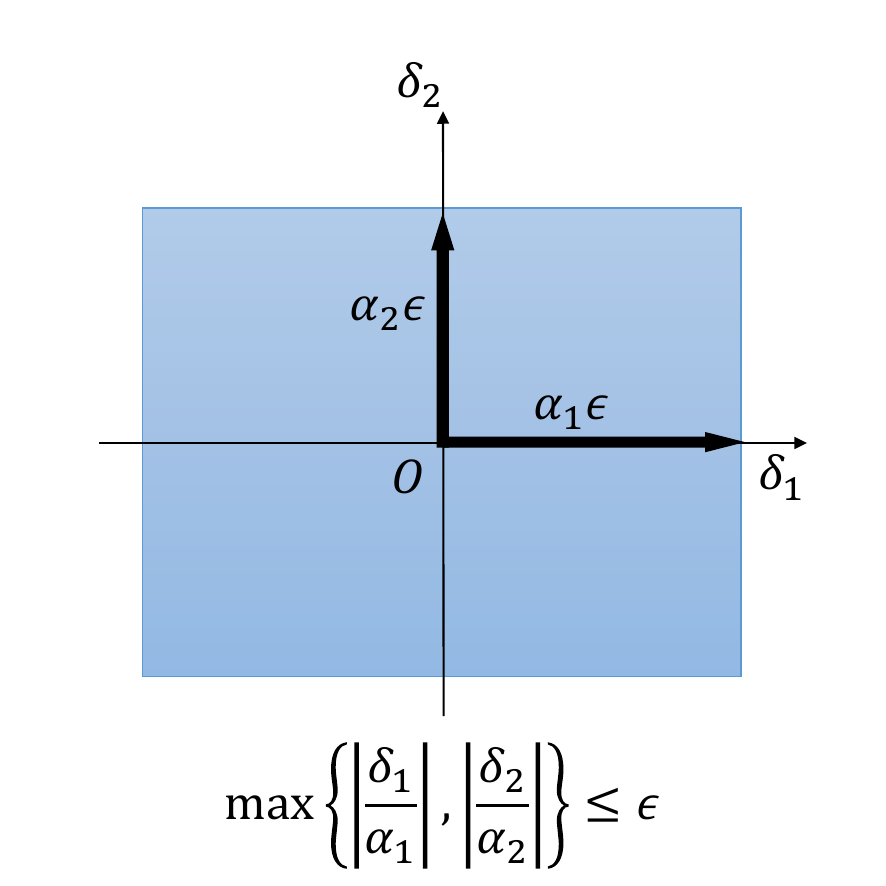}}
\caption{
Illustration of the regions that $L_2$-norm and $L_{+\infty}$ constraint specify and regions after rescaling.
}
\label{fig:norm}
\end{figure*}

\section{Methodology}

In this section, we first rethink the constraint set in adversarial training. Then, we propose to rescale different dimensions of the perturbation in the constraint set. Last, we propose the adaptively scaled adversarial training (ASAT) algorithm.

\subsection{Rescaling the Constraint in Adversarial Training}

In this section, we rethink the constraint in adversarial training in time series and propose to rescale the constraint.

\subsubsection{Rethinking the Constraint in Adversarial Training} 

Adversarial training algorithms are widely adopted in the CV~\citep{Intriguing_properties_of_neural_networks,Explaining_and_Harnessing_Adversarial_Examples,Towards_Evaluating_the_Robustness_of_Neural_Networks,YOPO}  and  NLP~\citep{Robsut_Translation_Doubly_Adversarial,freeLB,Robust_Machine_Translation} fields. The widely-adopted $L_p$-norm constraint, $S=\{\bm\delta:\|\bm\delta\|_p\le\epsilon\}$, is invariant about different dimensions of the perturbation, which is reasonable in both the CV and NLP fields. 

In the CV field, an important hypothesis of convolution neural networks is shift-invariant~\citep{Convolutional_Networks_Shift-Invariant_Again,Truly_shift-invariant_convolutional_neural_networks} and different pixels in an image can be treated of equal significance. In the NLP field, adversarial attacks can be conducted on the word embeddings and different dimensions of word embeddings are of similar significance~\citep{hotflip,freeLB}.

However, in time series, different dimensions of inputs are not symmetrical. Therefore, directly adopting the $L_p$-norm constraint in adversarial training in time series is not reasonable. 

\subsubsection{Rescaling the Constraint in Adversarial Training}

Scales of all dimensions of the $L_p$-norm bounded constraint are the same. However, in time series, dimensions with different timestamps may be of different significance. The time series is usually defined as $\vect{X}=(\vect{v}^{(1)}, \vect{v}^{(2)}, \cdots, \vect{v}^{(T)})$, where $\vect{v}^{(t)}\in \mathbb{R}^V$ is the input vector at the $t$-th timestamp and a smaller $t$ means an earlier timestamp. For example, $T=2$, $\vect{v}^{(1)}$ is the input vector at the first timestamp \textit{9:00, Jan. 21}, and $\vect{v}^{(2)}$ is the input vector at the second timestamp \textit{10:00, Jan. 21}. In our work, we flatten a time series into a vector $\vect{x}=(x_1, x_2, \cdots, x_k)^\text{T}\in \mathbb{R}^k$, where $k=TV$ and $x_{(t-1)V+j}=\text{v}^{(t)}_j$ is from the input vector with the $t$-th timestamp. Suppose $x_i$ is from the input vector with the $t_i$-th timestamp, namely $t_i=\lceil\frac{i}{V}\rceil$, where $\lceil\cdot\rceil$ is the ceil function.

Suppose $\bm\alpha=(\alpha_1, \alpha_2, \cdots, \alpha_k)^\text{T}$, we multiply the scale of dimension $i$ of the perturbation with $\alpha_i$ to rescale the radius of dimension $i$ from $\epsilon$ to $\alpha_i\epsilon$. As illustrated in Fig.~\ref{fig:norm}, the $L_2$-norm bounded constraint specifies a $k$-dimensional hyperball in $\mathbb{R}^k$ (a circle in $\mathbb{R}^2$, as shown in Fig.~\ref{fig:norm1}) and the $L_{+\infty}$-norm bounded constraint specifies a $k$-dimensional hypercube in $\mathbb{R}^k$ (a square in $\mathbb{R}^2$, as shown in Fig.~\ref{fig:norm2}). After rescaling, the hyperball is transformed into a hyperellipsoid (an ellipse in $\mathbb{R}^2$, as shown in Fig.~\ref{fig:norm3}) and the hypercube is transformed into a hypercuboid (a rectangle in $\mathbb{R}^2$, as shown in Fig.~\ref{fig:norm4}). The rescaling process assigns different importance to different dimensions of perturbations in adversarial training.  Generally, the $L_p$-norm constraint $\|\bm\delta\|_p\le\epsilon$ is transformed to the constraint:
\begin{align}
\|\bm\alpha^{-1}\odot\bm\delta\|_p=\left(\sum\limits_{i=1}^{k}\big|\frac{\delta_i}{\alpha_i}\big|^p\right)^\frac{1}{p}\le\epsilon.
\end{align}

\subsubsection{Fast Gradient Methods and PGD Algorithms after Rescaling}

After rescaling, the constraint set is $S=\{\bm\delta:\|\bm\alpha^{-1}\odot\bm\delta\|_p\le\epsilon\}$. For the fast gradient method target, namely maximizing the inner product of $\bm \delta$ and the gradient $\vect{g}_{k-1}$, the solution Eq.(\ref{eq:fast_gradient}) becomes:
\begin{align}
\bm\delta = \argmax_{\bm\delta\in S}\bm\delta^{\text{T}}\vect{g}_{k-1}=\epsilon\big(\bm\alpha\odot\text{sgn}(\vect{g}_{k-1})\big)\odot\frac{|\bm\alpha\odot\vect{g}_{k-1}|^\frac{1}{p-1}}{\||\bm\alpha\odot\vect{g}_{k-1}|^\frac{1}{p-1}\|_p}.
\label{eq:fast_gradient_new}
\end{align}

For the projection function $\Pi_S(\vect{v})$ that projects $\vect{v}$ into the set $S$,  we adopt the variants of Eq.(\ref{eq:projection_L2}) and Eq.(\ref{eq:projection_Linf}):
\begin{align}
\Pi_{\{\vect{v}:\|\bm\alpha^{-1}\odot\vect{v}\|_2\le\epsilon\}}{(\vect{v})} =& \min\{\|\bm\alpha^{-1}\odot\vect{v}\|_2,\epsilon\}\frac{\vect{v}}{\|\bm\alpha^{-1}\odot\vect{v}\|_2} \label{eq:projection_new_L2},\\
\Pi_{\{\vect{v}:\|\bm\alpha^{-1}\odot\vect{v}\|_{+\infty}\le\epsilon\}}{(\vect{v})} &= \bm\alpha\odot\text{clip}(\bm\alpha^{-1}\odot\vect{v},-\epsilon, \epsilon) \label{eq:projection_new_Linf}.
\end{align}

The proof and theoretical details are in Appendix.

\begin{algorithm}[!t]
   \caption{Multi-step Risk Averaging Adversarial Training Algorithm with Dimension-wise Scales}
   \label{alg:freeLB}
\begin{algorithmic}[1]
    \REQUIRE Network $f$; parameters $\bm \theta \in\mathbb{R}^k$; loss $\mathcal{L}$ and dataset $\mathcal{D}$; constraint set $S$; step number $K$ and step size $\tau$; optimizer $\mathcal{O}$; iterations; batch size $|\mathcal{B}|$.
    \STATE Initialize $\bm \theta$.
    \WHILE {Training}
    \STATE Generate a batch $\mathcal{B}$.
    \FOR {$(\vect{x}, y) \in \mathcal{B}$}
    \STATE $\bm\delta_0\gets \vect{0}_k.$
    \STATE  Calculate the risk $\mathcal{L}(f(\vect{x}+\bm\delta_{0}, \bm \theta), y)$ and the gradient $\vect{g}_{0}=\vect{g}=\nabla_{\vect{x}}\mathcal{L}( f(\vect{x},\bm\theta), y)$.
    \STATE Calculate time-dependent or adaptive scales $\alpha_i$.
    \FOR {k = 1 to $K$}
    \STATE Generate $\vect{u}_{k}=\argmax_{\vect{u}\in S_1}\vect{u}^{\text{T}}\vect{g}_{k-1}$, where $S_1=\{\vect{u}:\|\bm\alpha^{-1}\odot\vect{u}\|_p\le \tau\}$, according to Eq.(\ref{eq:fast_gradient_new}) with step size $\tau$.
    \STATE Generate $\bm\delta_{k}= \Pi_S(\bm\delta_{k-1} + \vect{u}_{k})$, where $S=\{\bm\delta:\|\bm\alpha^{-1}\odot\bm\delta\|_p\le \epsilon\}$,according to Eq.(\ref{eq:PGD}), Eq.(\ref{eq:projection_new_L2}) and Eq.(\ref{eq:projection_new_Linf}).
    \STATE  Calculate the risk $\mathcal{L}(f(\vect{x}+\bm\delta_{k}, \bm \theta), y)$ and the gradient $\vect{g}_{k}=\nabla_{\vect{x} + \bm\delta_{k}}\mathcal{L}( f(\vect{x} + \bm\delta_{k},\bm\theta), y)$.
    \ENDFOR
    \ENDFOR
    \STATE Update $\bm\theta$ with optimizer $\mathcal{O}$ and gradients $\{\vect{g}_{k}\}_{k=0}^{K}$ according to Eq.(\ref{eq:FreeLB}).
    \ENDWHILE
\end{algorithmic}
\end{algorithm}

\subsection{Time-dependently Scaled and Adaptively Scaled Adversarial Training}

In this section, we propose two lines of adversarial training with time-dependent or adaptive scales. 

\subsubsection{Multi-step Risk Averaging Adversarial Training}
To take full advantage of the gradients of the adversarial examples generated at all steps, we adopt a multi-step risk averaging adversarial training as illustrated in Algorithm~\ref{alg:freeLB}. The target of the proposed algorithm is minimizing the average risk of both the clean loss and the adversarial sensitivities in multiple steps. Following \cite{freeLB}, we optimize the risk on a batch $\mathcal{B}$ by minimizing the average risks of multiple adversarial examples generated during the PGD process:
\begin{align}
\bm\theta=\argmin\limits_{\bm\theta}\mathbb{E}_{(\vect{x},y)\in\mathcal{B}}\big[\frac{1}{K+1}\sum\limits_{k=0}^K\mathcal{L}(f(\vect{x}+\bm\delta_{k}, \bm \theta), y)\big].
\label{eq:FreeLB}
\end{align}

\subsubsection{Time-dependently Scaled Adversarial Training}

Adaptive scales $\alpha_i$ are determined by the $t_i$, we require $\alpha_i\in [0, 1]$. Since in time series analysis, data of closer timestamps with larger $t_i$ may be more important, We may assume $\alpha_i\le\alpha_j$ when $t_i<t_j$, and $\max\limits_{1\le i\le k}\alpha_i=1$. The scale $\alpha_i$ decays when $t_i$ decreases.

We consider three simple decay functions:
\begin{itemize}
    \item \textbf{Constant (Const).} $\alpha_i=1$.
    \item \textbf{Exponential Decay (Exp).} $\alpha_i=\gamma^{T-t_i}$.
    \item \textbf{Linear Decay (Linear).} $\alpha_i=1-\gamma(T-t_i)$.
\end{itemize}
where the hyperparameter $\gamma$ controls the decaying speed, and $T$ denotes the number of different timestamps. Traditional adversarial training can be seen as adversarial training with constant scales.

\subsubsection{Adaptively Scaled Adversarial Training}
\label{sec:asat}

Time-dependent scales $\alpha_i$ are determined by the $i$-th dimension in the data instance $x_i$ and its gradient $g_i=\nabla_{x_i}\mathcal{L}( f(\vect{x},\bm\theta), y)$, namely $x_i$ is the $i$-th dimension of $\vect{x}$ and $g_i$ is the $i$-th dimension of $\vect{g}=\nabla_{\vect{x}}\mathcal{L}(f(\vect{x},\bm\theta), y)$. Similarly to time-dependent scales, we may assume that $0\le \alpha_i \le 1$ and $\max\limits_{1\le i\le k}\alpha_i=1$.

The ideal attack scales should relate to the scales of data instance, which will be analyzed later in detail in Theorem~\ref{thm:generalization_error} and Sec.~\ref{sec:analysis} that the ideal attack scales $\bm\beta$ should be proportional to the shifts between the training and test data. We may assume that the ideal attack scales or the data shifts are proportional to the scale of data instance $|x_i|$, namely:
\begin{align}
    \alpha_i\propto |x_i|\label{eq:data}.
\end{align}

As analyzed in Proposition~\ref{prop:risk_non_linear}, when conducting adaptive scales, the dimension-wise adversarial sensitivities are $\mathcal{R}^{(i)}_\text{adv}(\alpha_i\epsilon)=\alpha_i\epsilon\big(\mathbb{E}|g_i|\big)+ O\big(\frac{M\epsilon^2}{2}\big)\propto \alpha_i\mathbb{E}|g_i|$. In traditional adversarial training with constant scales $\alpha_i=1$, dimensions with larger $|g_i|$ have higher dimension-wise adversarial sensitivities. With constant scales, the loss change caused by dimensions with larger $|g_i|$ can easily cause the learning collapse while dimensions with smaller $|g_i|$ only cause the loss to increase a little. Neural networks pay more attention to dimensions with higher $|g_i|$ and may overfit these dimensions. To settle this issue and alleviate potential overfitting, we propose to balance the dimension-wise adversarial sensitivities of different dimensions, namely:
\begin{align}
    \alpha_i\propto \frac{1}{|g_i|},
\end{align}
combined with Eq.(\ref{eq:data}), we may assume $\alpha_i\propto \frac{|x_i|}{|g_i|}$, namely $\alpha_i=\frac{|x_i|}{C|g_i|}$, where $C=\max\limits_{1\le i\le k}\frac{|x_i|}{|g_i|}$. In our implementation, we adopt:
\begin{align}
    \alpha_i=\frac{1}{\max\limits_{1\le i\le k}\frac{|x_i|+e}{|g_i|+e}}\frac{|x_i|+e}{|g_i|+e},
    \label{eq:scales}
\end{align}
where we add a small positive real number $e=10^{-6}$ on $|x_i|$ and $|g_i|$ for numerical stability. 

As shown in Algorithm~\ref{alg:freeLB}, we calculate adaptive scales $\alpha_i$ (line 7) after calculating the gradient $\vect{g}$ (line 6). Therefore, the calculation of adaptive scales $\alpha_i$ only involves Eq.(\ref{eq:scales}) with  the calculated $\vect{g}$, and does not require extra neural network forward or backward propagation compared to traditional adversarial training. The extra computational cost is negligible compared to the forward and backward propagation cost of neural networks.

\subsection{Theoretical Analysis}

In this section, we first analyze the adversarial sensitivity under our proposed adaptively scaled adversarial attacks. Then, we reveal that the adversarial sensitivity under our proposed adaptively scaled adversarial attacks determines the generalization bound when the training and test distributions are not IID. 

\subsubsection{Adversarial Sensitivity after Rescaling}
We show the adversarial sensitivity or risk under our proposed adaptively scaled adversarial attacks after rescaling in Theorem~\ref{thm:risk_rescaling}.

\begin{thm}
Consider a neural network that is convex and $M$-smooth in the neighborhood of $\bm\theta$, suppose $\vect{g}$ denotes the gradient, $\vect{g}=\nabla_{\vect{x}}\mathcal{L}( f(\vect{x},\bm\theta), y)$, then its adversarial sensitivity or risk under our proposed adaptively scaled adversarial attacks where $S=\{\bm\delta:\|\bm\alpha^{-1}\odot\bm\delta\|_p\le\epsilon\}$ is:
\begin{align}
\mathcal{R}_\text{adv}(S)=\big(1+o(1)\big)\epsilon\big(\mathbb{E}\|\bm\alpha\odot{\vect{g}}\|_\frac{p}{p-1}\big).
\end{align}
\label{thm:risk_rescaling}
\end{thm}

\subsubsection{Generalization Bound and Adversarial Sensitivity}
\label{sec:analysis}

Let $\vect{z}=(\vect{x}, y)$, we may assume the training distribution $\mathcal{P}:=p(\vect{z})$ and the test distribution $\mathcal{Q}:=q(\vect{z})$ are two close distributions but there exists a distributional shift between them. Assume $\phi:\mathbb{R}^k\to\mathbb{R}^k$ maps a training instance $\vect{z}$ into a test instance $\phi(\vect{z})$, namely $q(\phi(\vect{z}))|\mathcal{J}(\phi(\vect{z}))|=p(\vect{z})$, where $\mathcal{J}$ is the Jacobi determinant. Define $\mathcal{L}_\mathcal{P}:=\mathbb{E}_{(\vect{x}, y)\sim p(\vect{x}, y)}[\mathcal{L}(f(\vect{x},\bm\theta), y)]$, $\bm\beta:=\mathbb{E}_{\vect{z}\sim q(\vect{z})}[\vect{z}]-\mathbb{E}_{\vect{z}\sim p(\vect{z})}[\vect{z}]=\mathbb{E}_{\vect{z}\sim p(\vect{z})}[\phi(\vect{z})-\vect{z}]$, and $\mathcal{R}_\text{adv}(S)$ is the adversarial sensitivity under adaptively scaled adversarial attacks on $\mathcal{P}$. Define the $l$-th order distributional shift distance as:
\begin{align}
D_l:=\mathbb{E}_{\vect{z}\sim p(\vect{z})}\|\phi(\vect{z})-\vect{z}\|^l.
\end{align}

For example, when the training and test distributions are two multivariate Gaussian distributions, namely $\mathcal{N}(\bm{\mu}_p,\vect{\Sigma}_p)$ and $\mathcal{N}(\bm{\mu}_q,\vect{\Sigma}_q)$ respectively, then $\bm\beta=\bm{\mu}_q-\bm{\mu}_p$, $\phi(\vect{z})=\bm{\mu}_q+\vect{\Sigma}_q\vect{\Sigma}_p^{-1}(\vect{z}-\bm{\mu}_p)$. When training and test distributions are close, namely $\|\bm\mu_p-\bm\mu_q\|_2$ and $\|\vect{\Sigma}_q\vect{\Sigma}_p^{-1}-\bm{I}\|_2$ are bounded, $D_l:=\mathbb{E}_{\vect{z}\sim p(\vect{z})}\|\phi(\vect{z})-\vect{z}\|^l=\mathbb{E}_{\vect{z}\sim p(\vect{z})}\|(\bm{\mu}_q-\vect{\Sigma}_q\vect{\Sigma}_p^{-1}\bm{\mu}_p)+(\vect{\Sigma}_q\vect{\Sigma}_p^{-1}-\vect{I})\vect{z}\|^l=\sum\limits_{i=0}^{l} \binom{l}{i}\times\|\bm{\mu}_q-\vect{\Sigma}_q\vect{\Sigma}_p^{-1}\bm{\mu}_p\|^i\|\vect{\Sigma}_q\vect{\Sigma}_p^{-1}-\vect{I}\|_2^{l-i}\mathbb{E}_{\vect{z}\sim p(\vect{z})}\big(\|\vect{z}\|^{l-i}\big)$ are also bounded.

Based on previous PAC-Bayes bounds~\citep{Some_PAC-Bayesian_Theorems}, we provide a theoretical analysis in Theorem~\ref{thm:generalization_error}.

\begin{thm}
\label{thm:generalization_error}
Suppose the training distribution $\mathcal{P}:=p(\vect{z})$ and the test distribution $\mathcal{Q}:=q(\vect{z})$ are two close distributions that have definite first and second order distributional shift distances, namely $D_1<+\infty, D_2<+\infty$, for any loss function $\mathcal{L}$ that maps an instance to $[0, 1]$ and is convex and $M$-smooth in the neighborhood of $\bm\theta$, with probability 1-$\omega$ over the choice of the training set $\mathcal{D}\sim \mathcal{P}$, the following bound holds:
\begin{align}
\mathcal{L}_\mathcal{Q}\le
\mathcal{L}_\mathcal{D}+\frac{\big(1+o(1)\big)k^\frac{1}{p}\mathcal{R}_\text{adv}(S)}{\epsilon}+\frac{MD_2}{2}+\text{Rem},
\end{align}
where $S=\{\bm\delta:\|\bm\beta^{-1}\odot\bm\delta\|_p\le\epsilon\}$, and the remainder term is $\text{Rem}=\sqrt{\frac{\log\frac{|\mathcal{D}|}{\omega}+2\log|\mathcal{D}|}{2|\mathcal{D}|}}+\frac{1}{|\mathcal{D}|}$. 
\end{thm}

When $|\mathcal{D}|\to+\infty$, we have $\text{Rem}\to 0$, then the generalization bound is mainly determined by $\frac{\big(1+o(1)\big)k^\frac{1}{p}\mathcal{R}_\text{adv}(S)}{\epsilon}+\frac{MD_2}{2}$. Theorem~\ref{thm:risk_rescaling} and Theorem~\ref{thm:generalization_error} reveal that the generalization bound $\mathcal{L}_\mathcal{Q}$ relates to the adversarial sensitivity $\mathcal{R}_\text{adv}(S)$ under adaptively scaled adversarial attacks. Traditional adversarial training algorithms with constant scales optimize the adversarial sensitivities under $L_p$ constrained adversarial attacks, while the proposed ASAT optimizes the adversarial sensitivities under adaptively scaled adversarial attacks, which are more precise estimations of the generalization bound. It implies that adaptive scales may help alleviate overfitting and improve the generalization ability compared to traditional adversarial training. Since $\bm\beta$ is unknown, we adopt adaptive scales determined by data scales and gradients in Sec.~\ref{sec:asat}. 
 
\section{Experiments}

In this section, we first introduce the task and dataset details, then introduce the baseline models and experimental settings. Last, we report experimental results.

\subsection{Tasks and Datasets}

\subsubsection{Volume prediction task}

We choose the volume prediction task, the input data consists of log prices and log volumes of previous $12$ time slots and the same time slots in previous $20$ trading days. The prices include open, close, high, and low prices. The input data consists of $32\times 4$ log prices and $32\times 1$ log volumes. The target is to regress the log volume. In our experiments, if the data $x_l$ of the $l$-th dimension is from the $i$-th farthest time slot or the $j$-th farthest day in history, we set $t_l=i$ or $t_l=j$, respectively.
\subsubsection{Datasets and data preprocessing}

For our research, we adopt the \textbf{hourly} inter-day volume prediction dataset and the \textbf{five-minute} intra-day volume prediction dataset. The two datasets are extracted from the price and volume data of the Topix500 (price index of the 500 most liquid and highly market capitalized stocks in Tokyo Stock) between \textit{Jan. 2017} and \textit{Feb. 2018}. We adopt the data of \textit{2017} as the training set and development set, and the data of \textit{Jan. 2018} and \textit{Feb. 2018} as the test set. The training set and the development set are randomly split with a ratio of $3:1$. The statistics of datasets are in Table~\ref{tab:dataset}.

\begin{table}[!ht]
\caption{Statistical information on the two datasets.}
\scriptsize
\label{tab:dataset}
\setlength{\tabcolsep}{5pt}
\centering
\begin{tabular}{@{}lcccccc@{}}
\toprule
 \bf Dataset & \multicolumn{3}{c}{\bf Hourly} & \multicolumn{3}{c}{\bf Five-minute}\\ 
 \midrule
 \bf Split & Train &  Dev  & Test & Train & Dev & \bf Test \\ 
 \midrule
\bf Samples & 49,728 & 16,554 & 26,818 & 106,784 & 35,552 &  27,584\\
 \bottomrule
\end{tabular}
\end{table}

For example, in the hourly dataset, if the ground truth is the log volume of \textit{10:00-11:00, Jan. 21}, whose timestamp is defined as the begin timestamp of the slot \textit{10:00, Jan. 21}, then the previous 12 time slots are: \textit{9:00-10:00, Jan. 21}, \textit{14:00-15:00, Jan. 20}, etc, and the history 20 days are: \textit{10:00-11:00, Jan. 20}, \textit{10:00-11:00, Jan. 19}, etc. It is similar to the data preprocessing of the five-minute dataset. We deleted the data instances consisting of missing volumes or prices. In the five-minute dataset, the previous 12 time slots are collected from the same day to the ground truth. 

\subsubsection{Evaluation metrics}

We adopt four evaluation metrics: mean squared error (MSE), root mean squared error (RMSE), mean absolute error (MAE), and accuracy (ACC). Suppose $\hat y=f(\vect{x}, \bm\theta)$ and $y$ is the ground truth, then these metrics are defined as:
\begin{align}
    \text{MSE}&=\mathbb{E}_{(\vect{x}, y)\sim\mathcal{D}}(\hat y - y)^2,\\
    \text{RMSE}&=\sqrt{\mathbb{E}_{(\vect{x}, y)\sim\mathcal{D}}(\hat y - y)^2},\\
    \text{MAE}&=\mathbb{E}_{(\vect{x}, y)\sim\mathcal{D}}|\hat y - y|,\\
    \text{ACC}&=\mathbb{P}_{(\vect{x}, y)\sim\mathcal{D}}\big((\hat y - x_\text{last})\times(y - x_\text{last})>0\big),
\end{align}
here $x_\text{last}$ denotes the volume of the last time slot, and ACC is the accuracy of whether the volume increases or decreases compared to the last time slot.

\subsection{Baselines}

We implement multiple moving average baselines and three neural network baselines in our experiments.

\subsubsection{Moving average baselines}

We adopt moving average baselines as our baselines, which are commonly used in technical analysis in finance. In statistics, suppose a series of volumes $v_1, v_2, \cdots, v_T$, three commonly adopted estimations are:
\begin{itemize}
    \item \textbf{Naive forecasting.} The naive forecasting algorithm uses $v_T$ as the prediction. In our experiments, we try to adopt the volumes of \textbf{yesterday} or \textbf{last time slot} as the prediction.
    \item \textbf{Simple moving average (SMA).} SMA uses the average value $\bar v=\frac{1}{T}\sum\limits_{i=1}^T v_i$ as the prediction. In our experiments, we try to adopt the volumes of \textbf{20-day average} or \textbf{12-slot average} as the prediction.
    \item \textbf{Exponential moving average (EMA).} EMA places a greater weight on the nearest values. It sets $x_1=v_1$, $x_t=(1-\rho)x_{t-1}+\rho v_t$ and uses $x_T$ as the prediction. In our experiments, we adopt $\rho=0.04$ and try to use the \textbf{20-day EMA} or \textbf{12-slot EMA} as the prediction.
\end{itemize}

We also try a method that considers both the 20-day average and 12-slot average volumes, \textbf{20-day and 12-slot average}. It uses the mean of the 20-day average and the 12-slot average as the prediction.

\subsubsection{Linear}
The linear model is formulated as $f(\vect{x}, \bm\theta)=\bm\theta^\text{T}\vect{x}+b$, where $\vect{x}\in\mathbb{R}^{160}$ is the flattened vector of the concatenated input of both 12-slot and 20-day history.

\subsubsection{LSTM}
The long-short term memory (LSTM)~\citep{LSTM} networks can capture features and long-term dependency in the entire sequences of data. Following \cite{volume-lstm}, we implement two one-layer LSTM models with the global attention mechanism~\citep{attention} to generate the representation vectors of previous time slots and 20-day history. Before feeding input into LSTM, we first adopt two linear layers to project the inputs into two higher dimensional spaces, respectively. Then, we concatenate the representation vectors and feed them into a linear layer to regress the log volume. The input size, hidden size, and output size of LSTM models are $200$.

\subsubsection{Transformer}
We also implement a six-layer Transformer~\citep{transformer} encoder model as a baseline. We concatenate the 12-slot and 20-day history input and add a special token \text{[CLS]} before it to get a series $\vect{X}\in\mathbb{R}^{33\times 5}$. The data of \text{[CLS]}, $\vect{X}[0,:]$, is treated as a trainable parameter of Transformer and no adversarial attacks are conducted on this dimension. Before feeding input into Transformer, we first adopt one linear layer to project the inputs into a higher dimensional space. Then, we feed the output representation of \text{[CLS]} of the last layer into a linear layer to regress the volume. The input size, hidden size, and output size are $200$, where the hidden states are split into $8$ heads. 

\begin{table*}[!t]
\caption{Experimental results of moving average baselines, neural network baselines, and adversarial training algorithms. AT denotes adversarial training, and ASAT denotes adaptively scaled adversarial training.}
\scriptsize
\label{tab:results}
\setlength{\tabcolsep}{2pt}
\centering
\begin{tabular}{@{}lcccccccc@{}}
\toprule
 \bf Dataset & \multicolumn{4}{c}{\bf Hourly} & \multicolumn{4}{c}{\bf Five-minute}\\ 
 \midrule
 \bf Model & \bf MSE & \bf RMSE  & \bf MAE  & \bf ACC & \bf MSE & \bf RMSE  & \bf MAE  & \bf ACC\\ 
 \midrule
  yesterday & 0.286 & 0.535 & 0.406 & 0.689 & 1.205 & 1.098 & 0.797 & 0.666 \\
 20-day average & \textbf{0.249} & \textbf{0.499} & \textbf{0.385} & \textbf{0.710} & \textbf{0.700} & \textbf{0.837} & \textbf{0.608} & \textbf{0.709} \\
 20-day EMA & 0.288 & 0.536 & 0.413 & 0.696 & 0.808 & 0.899 & 0.659 & 0.692 \\
 last time slot& 0.324 & 0.569 & 0.439 & 0.500 & 1.125 & 1.060 & 0.745 & 0.500 \\
 12-slot average &  0.524 & 0.724 & 0.413 & 0.696 & 0.988 & 0.994 & 0.714 & 0.629 \\
 12-slot EMA & 0.331 & 0.575 & 0.443 & 0.642 & 1.314 & 1.146 & 0.844 & 0.602 \\
 20-day and 12-slot average & 0.271 & 0.521 & 0.393 & 0.662 & 0.741 & 0.861 & 0.609 & 0.695 \\
 \midrule
 Linear & 0.227$\pm$0.019 & 0.477$\pm$0.019 & 0.370$\pm$0.023 & 0.708$\pm$0.012 & 0.808$\pm$0.072 & 0.892$\pm$0.041 & 0.700$\pm$0.029 & 0.660$\pm$0.037 \\
 +Traditional AT  & 0.207$\pm$0.003 & 0.454$\pm$0.003 & 0.348$\pm$0.007 & \textbf{0.723$\pm$0.004} & 0.677$\pm$0.055 & 0.824$\pm$0.034 & 0.625$\pm$0.042 & 0.691$\pm$0.025   \\
 +Time-dependent AT (Linear) &  \textbf{0.206$\pm$0.003} & \textbf{0.454$\pm$0.003} & \textbf{0.347$\pm$0.007} & 0.722$\pm$0.005 & 0.678$\pm$0.055 & 0.822$\pm$0.034 & 0.624$\pm$0.042 & 0.691$\pm$0.025   \\
 +Time-dependent AT (Exp) &  0.208$\pm$0.003 & 0.456$\pm$0.003 & 0.349$\pm$0.007 & 0.720$\pm$0.005 & 0.676$\pm$0.055 & 0.823$\pm$0.034 & 0.624$\pm$0.042 & 0.692$\pm$0.025 \\
 +ASAT &  \textbf{0.206$\pm$0.004} & \textbf{0.454$\pm$0.004} & \textbf{0.347$\pm$0.006} & 0.720$\pm$0.007 & \textbf{0.671$\pm$0.050} & \textbf{0.819$\pm$0.031} & \textbf{0.620$\pm$0.039} & \textbf{0.694$\pm$0.023}  \\
 \midrule
 LSTM & 0.223$\pm$0.005 & 0.472$\pm$0.006 & 0.361$\pm$0.007 & 0.708$\pm$0.004 & 0.801$\pm$0.046 & 0.895$\pm$0.027 & 0.690$\pm$0.013 & 0.671$\pm$0.006 \\
 +Traditional AT  &  0.221$\pm$0.002 & 0.470$\pm$0.002 & 0.360$\pm$0.003 & 0.711$\pm$0.003 & 0.820$\pm$0.054 & 0.905$\pm$0.030 & 0.692$\pm$0.030 & 0.673$\pm$0.012  \\
 +Time-dependent AT (Linear) & 0.222$\pm$0.002 & 0.471$\pm$0.002 & 0.360$\pm$0.004 & 0.710$\pm$0.002 & 0.805$\pm$0.052 & 0.897$\pm$0.029 & 0.688$\pm$0.032 & 0.673$\pm$0.014 \\
 +Time-dependent AT (Exp)  & 0.221$\pm$0.004 & 0.470$\pm$0.004 & 0.360$\pm$0.004 & 0.711$\pm$0.004 & 0.798$\pm$0.050 & 0.893$\pm$0.028 & 0.683$\pm$0.030 & \textbf{0.675$\pm$0.013} \\
 +ASAT &  \textbf{0.220$\pm$0.002} & \textbf{0.469$\pm$0.003} & \textbf{0.359$\pm$0.005} & \textbf{0.712$\pm$0.003} & \textbf{0.792$\pm$0.058} & \textbf{0.890$\pm$0.032} & \textbf{0.681$\pm$0.033} & \textbf{0.675$\pm$0.013} \\
 \midrule
 Transformer & 0.219$\pm$0.012 & 0.470$\pm$0.013 & 0.357$\pm$0.012 & 0.711$\pm$0.012 & 0.665$\pm$0.044 & 0.815$\pm$0.027 & 0.610$\pm$0.032 & 0.695$\pm$0.019 \\
 +Traditional AT  & 0.219$\pm$0.014 & 0.467$\pm$0.015 & 0.359$\pm$0.013 & 0.711$\pm$0.013 & 0.658$\pm$0.033 & 0.811$\pm$0.020 & 0.603$\pm$0.024 & 0.702$\pm$0.012 \\
 +Time-dependent AT (Linear) & 0.208$\pm$0.005 & 0.456$\pm$0.005 & 0.350$\pm$0.009 & 0.718$\pm$0.008 & 0.666$\pm$0.071 & 0.815$\pm$0.042 & 0.611$\pm$0.051 & 0.693$\pm$0.028\\
 +Time-dependent AT (Exp) &  0.207$\pm$0.005 & 0.455$\pm$0.006 & 0.348$\pm$0.009 & 0.721$\pm$0.007 & 0.658$\pm$0.060 & 0.810$\pm$0.036 & 0.606$\pm$0.046 & 0.696$\pm$0.026\\
 +ASAT &  \textbf{0.201$\pm$0.002} & \textbf{0.449$\pm$0.002} & \textbf{0.342$\pm$0.004} &\textbf{0.726$\pm$0.003} & \textbf{0.626$\pm$0.023} & \textbf{0.791$\pm$0.015} & \textbf{0.580$\pm$0.017} & \textbf{0.712$\pm$0.010} \\
\bottomrule
\end{tabular}
\end{table*}

\subsection{Settings and Choices of Hyperparameters}

We train every model for $5$ epochs and report the test performance on the checkpoint with the lowest valid loss. We adopt the Adam optimizer and initialize the learning rate with $0.001$. The batch size is $32$. We repeat every experiment with $5$ runs. Experiments are conducted on NVIDIA TITAN RTX GPUs.

Following \cite{freeLB}, we try $K\in \{1, 2, 3\}$ and $\tau=1.5*\epsilon/K$, where the hyperparameter $K$ is the step number in both PGD attacks or the attacks in the multi-step risk averaging algorithm, $\tau$ is the step size and $\epsilon$ controls the strength of adversarial training in the constraint set $S$. We try $L_2$ and $L_{+\infty}$ in $S$, and explore the hyperparameters $\epsilon$ with the grid search method. The process of hyperparameter search shows that too large or small $\epsilon$ cannot improve the model accuracy well. Since too small $\epsilon$ is not enough for improving generalization ability and too large $\epsilon$ may harm the learning process, an appropriate $\epsilon$ needs to be selected. For time-dependently scaled adversarial training, an appropriate $\gamma$ needs to be selected with the grid search method, too. Detailed experimental results in the choice of hyperparameters are reported in Appendix.

\subsection{Experimental Results}

After finding the best configurations of hyperparameters, the experimental results are shown in Table~\ref{tab:results}. It can be concluded that the 20-day average baseline performs best among multiple moving average baselines. We only report results of the 20-day average baseline in the following analysis. Moreover, three neural network baselines perform better than the 20-day average baseline.

For adversarial training algorithms, the traditional adversarial training, the proposed time-dependently scaled adversarial training, and the adaptively scaled adversarial training algorithms can all improve the performance of baselines. Among them, adversarial training with time-dependent or adaptive scales outperforms traditional adversarial training, and the adaptively scaled adversarial training algorithm performs best among multiple adversarial training algorithms. We also conduct hypothesis tests to verify that the adaptively scaled adversarial training algorithm outperforms baselines statistically significantly $(p<0.05)$ on the linear and the Transformer model. In the LSTM model, the baseline is weak and the improvements of all adversarial training methods are marginal. Details of hypothesis tests are reported in Appendix.

\begin{table}[!t]
\caption{Results of ablation study. ASAT (PGD) denoted ASAT with the PGD target.}
\scriptsize
\label{tab:ablation}
\setlength{\tabcolsep}{1pt}
\centering
\begin{tabular}{@{}lcccccccc@{}}
\toprule
 \bf Dataset & \multicolumn{4}{c}{\bf Hourly} & \multicolumn{4}{c}{\bf Five-minute}\\ 
 \midrule
 \bf Model & \bf MSE & \bf RMSE  & \bf MAE  & \bf ACC & \bf MSE & \bf RMSE  & \bf MAE  & \bf ACC\\ 
 \midrule
 20-day average & 0.249 & 0.499 & 0.385 & 0.710 & 0.700 & 0.837 & 0.608 & 0.709 \\
 \midrule
 Linear & 0.227 & 0.477& 0.370 & 0.708 & 0.808 & 0.892 & 0.700 & 0.660 \\
 +Traditional AT & 0.207 & \textbf{0.454} & 0.348 & \textbf{0.723} & 0.677 & 0.824 & 0.625 & 0.691 \\
 +ASAT & {0.206} & \textbf{0.454} & \textbf{0.347} & 0.720 & \textbf{0.671} & \textbf{0.819} & \textbf{0.620} & \textbf{0.694}  \\
 +$\alpha_i=\frac{|x_i|}{\max|x_i|}$ & \textbf{0.205} & \textbf{0.454} & 0.348 & 0.720 &0.691 & 0.832 &0.637 &0.682 \\
 +$\alpha_i=\frac{|g_i|^{-1}}{\max|g_i|^{-1}}$ & 0.207 & 0.455 & 0.348 & 0.718& 0.679 & 0.824 & 0.625 & 0.690 \\
 +ASAT (PGD) & 0.207 &0.455 &0.348& 0.719& 0.675 &0.821&0.622&0.692 \\
 \midrule
 LSTM & 0.223 & 0.472 & 0.361 & 0.708 & 0.801 & 0.895 & 0.690 & 0.671 \\
 +Traditional AT & 0.221 & 0.470 & 0.360 & 0.711 & 0.820& 0.905 & 0.692 & 0.673 \\
 +ASAT & \textbf{0.220} & \textbf{0.469} & \textbf{0.359} & \textbf{0.712} & \textbf{0.792} & \textbf{0.890} & \textbf{0.681} & \textbf{0.675} \\
 +$\alpha_i=\frac{|x_i|}{\max|x_i|}$ & 0.223 & 0.472 & 0.361 & 0.710 & 0.787& 0.887 & 0.692 & 0.668\\
 +$\alpha_i=\frac{|g_i|^{-1}}{\max|g_i|^{-1}}$ &0.221&0.470&0.360&0.710&0.800&0.894&0.685&0.674\\
 +ASAT (PGD) & \textbf{0.220} & \textbf{0.469}&0.360&0.711&0.817&0.903&0.692&0.673\\
 \midrule
 Transformer & 0.219 & 0.470 & 0.357 & 0.711 & 0.665 & 0.815 & 0.610 & 0.695 \\
 +Traditional AT  & 0.219& 0.467 & 0.359 & 0.711& 0.658 & 0.811 & 0.603 & 0.702 \\
 +ASAT &  \textbf{0.201} & \textbf{0.449} & \textbf{0.342} &\textbf{0.726} & \textbf{0.626} & \textbf{0.791} & \textbf{0.580} & \textbf{0.712} \\
 +$\alpha_i=\frac{|x_i|}{\max|x_i|}$ & 0.203& 0.450& 0.343&	0.7251 &0.632&0.795&0.585&0.711\\
 +$\alpha_i=\frac{|g_i|^{-1}}{\max|g_i|^{-1}}$ & 0.208&	0.456 &0.348&0.720&0.662&0.813&0.609&0.697\\
 +ASAT (PGD) & 0.206&0.453&	0.350&	0.720&0.637&0.798&0.590&0.708\\
\bottomrule
\end{tabular}
\end{table}

\section{Analysis}
In this section, we first conduct an ablation study. Then, we compare the adversarial robustness of neural networks under adversarial training. We also adopt ASAT to enhance existing methods. Last, we probe the sensitivities of different input dimensions and analyze the decision bases of neural networks, both on the dataset level and single instance level. 

\subsection{Ablation Study}

We conduct an ablation study to investigate the influence of different mechanisms in our proposed ASAT. In ASAT, as analyzed in Sec.~\ref{sec:asat}, scales $\alpha_i$ are determined by both $|x_i|$ and $|g_i|$ in Eq.(\ref{eq:scales}). We try both $\alpha_i={|x_i|}/{\max|x_i|}$ and  $\alpha_i={|g_i|^{-1}}/{\max|g_i|^{-1}}$ to verify the effectiveness of the proposed adaptive scales. The proposed ASAT algorithm minimizes the multi-step risk averaging target in Eq.(\ref{eq:FreeLB}). We try to adopt the traditional adversarial learning target in Eq.(\ref{eq:PGD_target}) to verify the effectiveness of the proposed multi-step risk averaging. Experimental results are shown in Table~\ref{tab:ablation}.

\begin{table}[!t]
\caption{Results of models under adversarial attacks. Risk $L_p$ denotes the adversarial risk under the $L_p$ constraint $\mathcal{R}_\text{adv}(\{\bm\delta: \|\bm\delta\|_p\le \epsilon \})$, Mean $\mathcal{R}^{(i)}$ denotes the mean dimension-wise adversarial risk $\frac{1}{k}\sum\limits_{i=1}^{k}\mathcal{R}_\text{adv}^{(i)}(\epsilon)$. We adopt $\epsilon=10^{-3}$ in experiments.}
\scriptsize
\label{tab:robustness}
\setlength{\tabcolsep}{3pt}
\centering
\begin{tabular}{@{}lccccc@{}}
\toprule
 \bf Dataset & \multicolumn{4}{c}{\bf Five-minute} \\
 \midrule
 \bf Model & \bf Clean MSE & \bf Risk $L_2$  &  \bf Risk $L_{+\infty}$ & \bf Mean $\mathcal{R}^{(i)}$ \\
  \midrule
  Linear & 0.808 & 2.9$\times 10^{-5}$ &	2.4$\times 10^{-4}$&1.5$\times 10^{-3}$	\\ 
 +Traditional AT & 0.677 &2.5$\times 10^{-5}$ &	2.0$\times 10^{-4}$ & 1.3$\times 10^{-3}$ \\ 
 +ASAT & 0.671 & 2.6$\times 10^{-5}$ &	1.9$\times 10^{-4}$ & 1.2$\times 10^{-3}$ \\ 
 \midrule
  LSTM & 0.801 &  2.0$\times 10^{-5}$&	5.6$\times 10^{-5}$&3.4$\times 10^{-4}$\\ 
 +Traditional AT & 0.820 & 1.8$\times 10^{-5}$&	4.8$\times 10^{-5}$&2.9$\times 10^{-4}$ \\ 
 +ASAT & 0.792 &  1.9$\times 10^{-5}$&	5.0$\times 10^{-5}$&3.2$\times 10^{-4}$\\ 
 \midrule
 Transformer & 0.665 & 1.6$\times 10^{-5}$&	6.8$\times 10^{-5}$&4.3$\times 10^{-4}$\\
 +Traditional AT & 0.658 & 1.4$\times 10^{-5}$&5.3$\times 10^{-5}$&3.3$\times 10^{-4}$\\ 
 +ASAT & 0.626 & 1.4$\times 10^{-5}$&4.6$\times 10^{-5}$&2.9$\times 10^{-4}$\\
 \bottomrule
\end{tabular}
\end{table}

Experimental results show that ASAT with the proposed adaptive scales outperforms ASAT with $\alpha_i={|x_i|}/{\max|x_i|}$ and ASAT with $\alpha_i={|g_i|^{-1}}/{\max|g_i|^{-1}}$ variants slightly, which demonstrates the necessity of both $|g_i|$ and $|x_i|$ in adaptive scales. ASAT with the multi-step risk averaging target outperforms ASAT with the PGD target, which demonstrates the effectiveness of the proposed multi-step risk averaging. Moreover, ASAT with the PGD target still outperforms traditional AT in most cases, which demonstrates that our proposed adaptive scales can improve the generalization ability of neural networks without the proposed multi-step risk averaging.

\begin{figure}[!t]
\centering
\subcaptionbox{MSE losses ($L_2$-norm).
\label{fig:robustness_transformer_a}}{\includegraphics[height=2.5 in,width=0.9\linewidth]{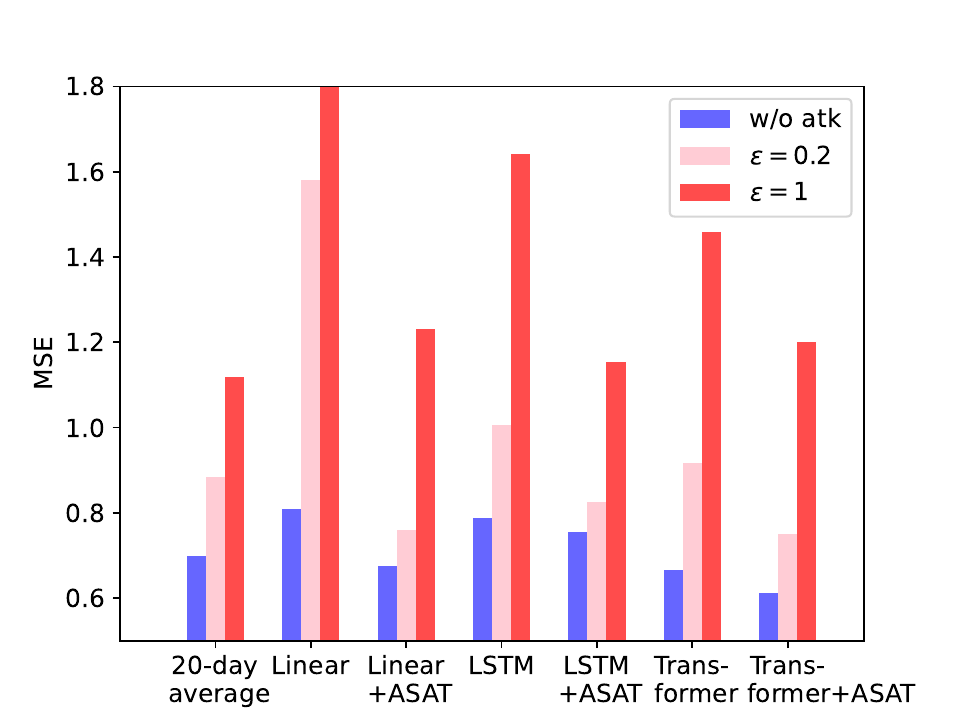}}
\subcaptionbox{MSE losses ($L_{+\infty}$-norm).
\label{fig:robustness_transformer_b}}{\includegraphics[height=2.5 in,width=0.9\linewidth]{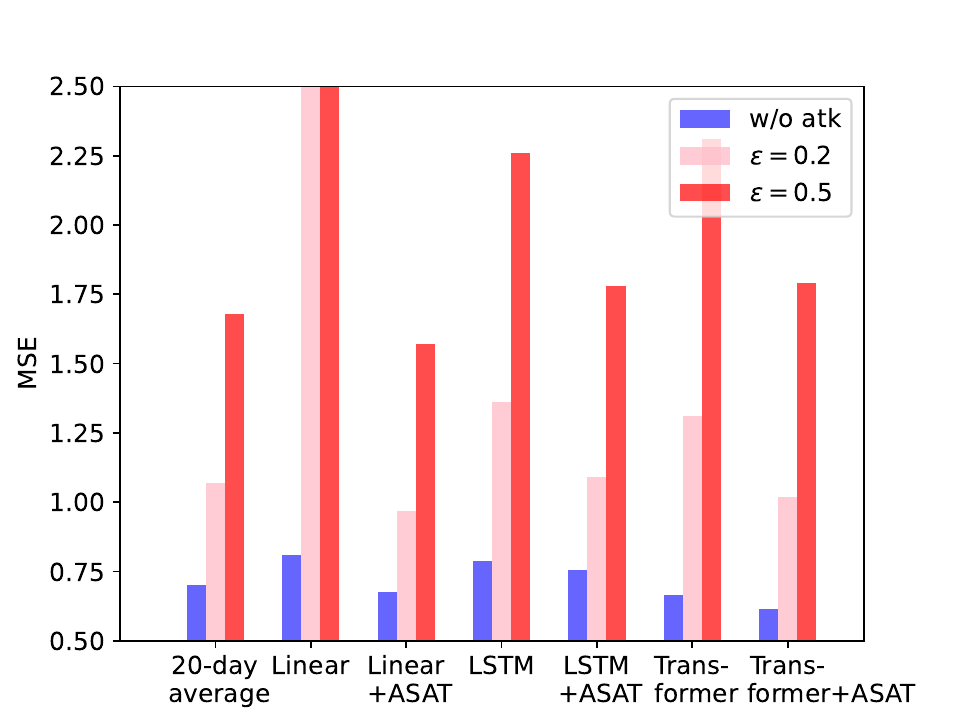}}
\caption{MSE losses of models under multiple adversarial attacks with $L_2$-norm and $L_{+\infty}$-norm constraints. MSE losses larger than 1.8 or 2.5 are not completely shown.
\label{fig:robustness_five}}
\end{figure}

\subsection{Adversarial Robustness with ASAT}

We compare the adversarial robustness of baseline models and models with adversarial training under multiple adversarial attacks on the five-minute dataset. 

We first adopt multiple adversarial sensitivity indicators, including adversarial sensitivities under $L_2$ or $L_{+\infty}$ attacks and mean dimension-wise adversarial sensitivity. Results are shown in Table~\ref{tab:robustness}. Experimental results show that ASAT can improve both the adversarial robustness and generalization ability of neural networks compared to the baselines. Compared to traditional adversarial training, ASAT can achieve better generalization ability and similar adversarial robustness.

We also visualize multiple adversarial attacks with several larger attacking scales. Experimental results are shown in Fig.~\ref{fig:robustness_five}. Similarly, we can conclude that ASAT can improve both the generalization ability and the adversarial robustness compared to baselines. Moreover, baseline linear, LSTM and Transformer models tend to be more sensitive to adversarial attacks than the 20-day average, which indicates that the traditional 20-day average baseline is more robust than neural networks. However, ASAT can improve the robustness of neural networks and achieve similar robustness to the 20-day average.

\begin{table}[!t]
\caption{Results of ASAT on existing high-performance methods, including DeepAR and M-LSTM.~\label{tab:stoa}}
\scriptsize
\setlength{\tabcolsep}{1pt}
\centering
\begin{tabular}{@{}lcccc@{}}
\toprule
 \bf Dataset & \multicolumn{4}{c}{\bf Hourly} \\
 \midrule
 \bf Model & \bf MSE & \bf RMSE  & \bf MAE & \bf ACC \\
 \midrule
  20-day average & {0.249} & {0.499} & {0.385} & {0.710} \\
  \midrule
  Transformer & 0.219$\pm$0.012 & 0.470$\pm$0.013 & 0.357$\pm$0.012 & 0.711$\pm$0.012 \\
   +ASAT & \textbf{0.201$\pm$0.002} & \textbf{0.449$\pm$0.002} & \textbf{0.342$\pm$0.004} &\textbf{0.726$\pm$0.003}  \\
  \midrule
 LSTM & 0.223$\pm$0.005 & 0.472$\pm$0.006 & 0.361$\pm$0.007 & 0.708$\pm$0.004  \\
 +ASAT & {0.220$\pm$0.002} & {0.469$\pm$0.003} & {0.359$\pm$0.005} & {0.712$\pm$0.003} \\
  \midrule
  M-LSTM & 0.235$\pm$0.013 & 0.485$\pm$0.013 & 0.373$\pm$0.012&0.687$\pm$0.013\\
   +ASAT & 0.225$\pm$0.010&0.474$\pm$0.011&0.363$\pm$0.007&0.698$\pm$0.008\\
  \midrule
  DeepAR & 0.207$\pm$0.006 & 0.455$\pm$0.007 & 0.348$\pm$0.011 & 0.719$\pm$0.010 \\
   +ASAT & \textbf{0.201$\pm$0.003}&\textbf{0.449$\pm$0.004}&0.343$\pm$0.003&0.723$\pm$0.004\\
 \bottomrule
 \end{tabular}
\end{table}

\begin{figure*}[!t]
\centering
\subcaptionbox{$\theta_i$ of baseline linear model-1.
\label{fig:mlp_a}}{\includegraphics[height=2.3 in,width=0.32\linewidth]{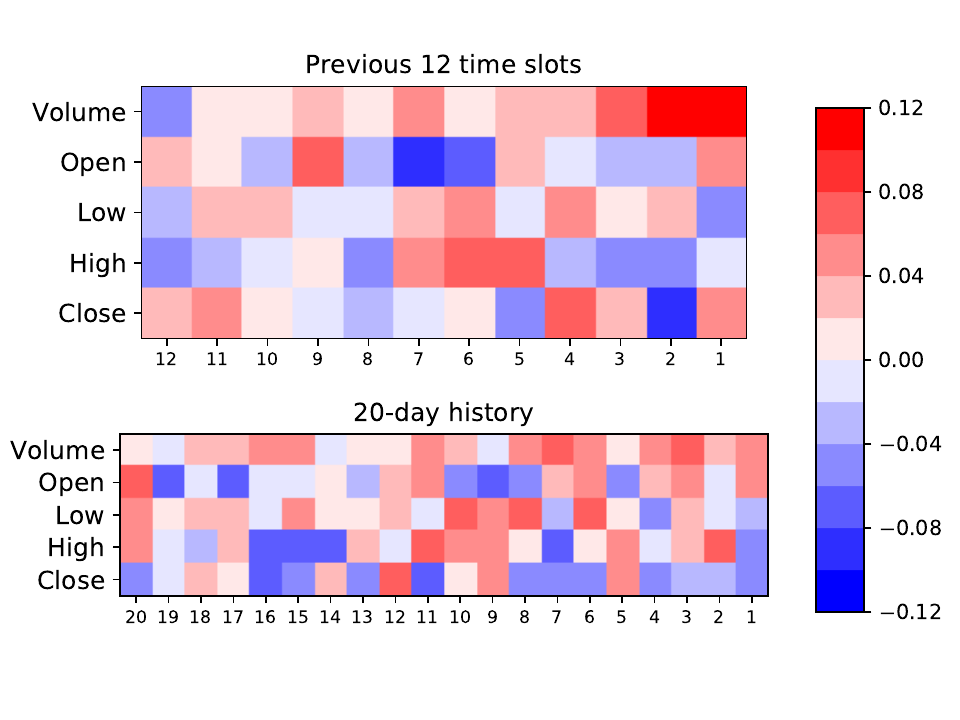}}
\subcaptionbox{$\mathcal{R}^{(i)}_\text{adv}$ of baseline linear model-1.\label{fig:mlp_b}}{\includegraphics[height=2.3 in,width=0.32\linewidth]{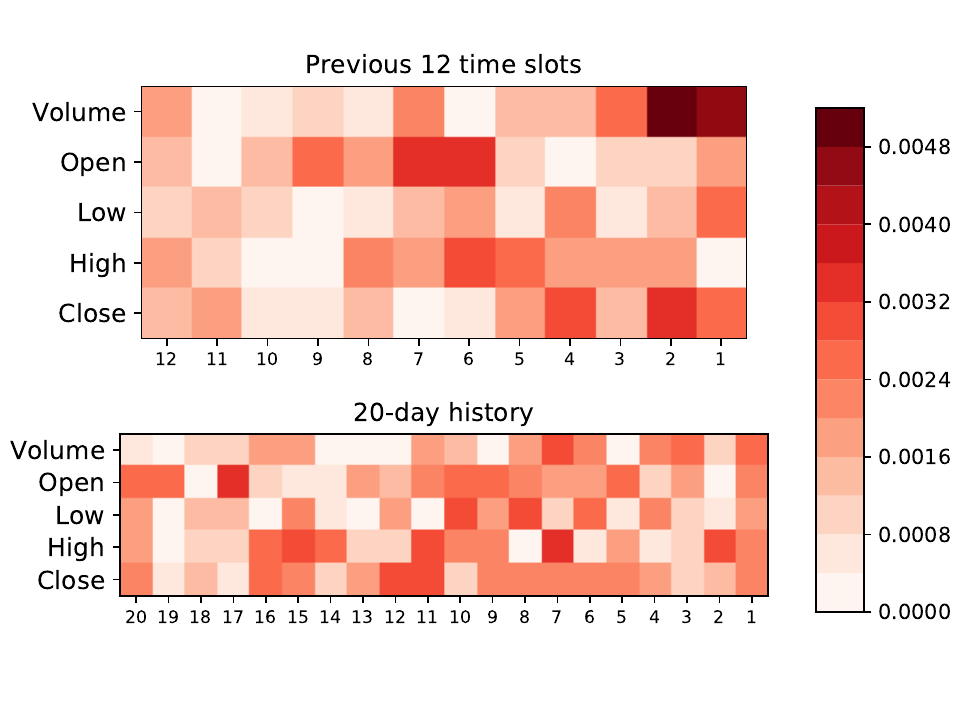}}
\subcaptionbox{$\mathcal{R}^{(i)}_\text{adv}$ of ASAT linear model-1.\label{fig:mlp_c}}{\includegraphics[height=2.3 in,width=0.32\linewidth]{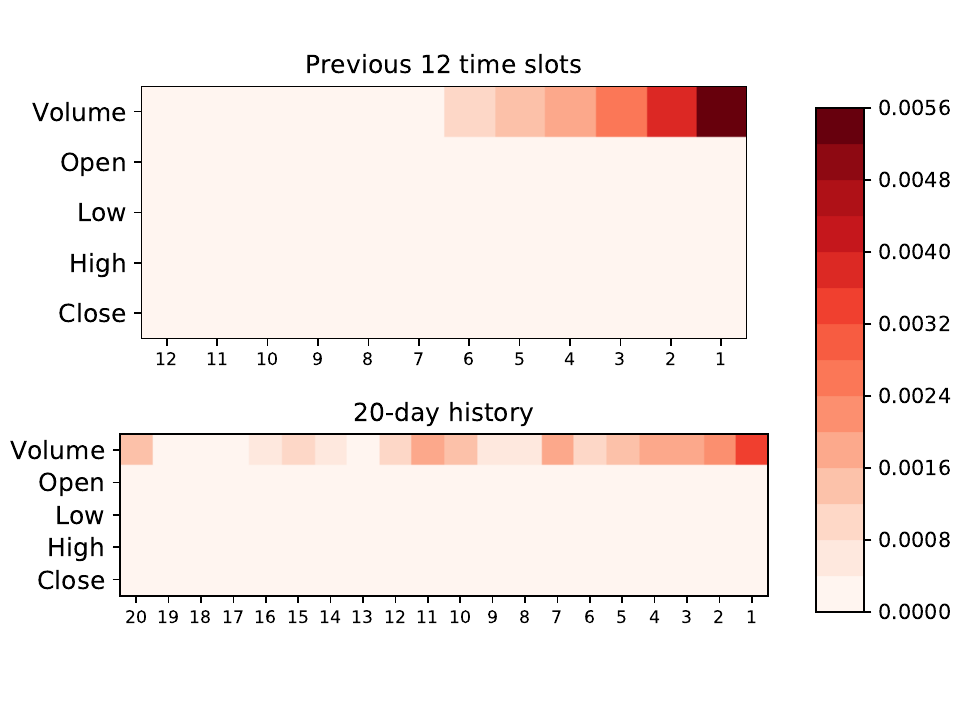}}
\caption{
Weights $\theta_i$ and dimension-wise sensitivities $\mathcal{R}^{(i)}_\text{adv}$ of linear models, both baseline and model with ASAT, on the dataset. The sensitivity of the linear model is proportional to the absolute value of the weight or the factor loading in a linear model.
}
\label{fig:mlp}
\end{figure*}

\subsection{ASAT for Enhancing Existing Methods}

We also implement two typical existing high-performance time-series prediction methods on the hourly dataset: Deep Auto-Regression (\textbf{DeepAR})~\citep{DeepAR} and Multiple time scale LSTM (\textbf{M-LSTM})~\citep{MLSTM}. The base models of both DeepAR and M-LSTM methods are LSTM models.

In DeepAR~\citep{DeepAR}, we choose the Gaussian distribution as the prior distribution and adopt two linear layers to predict $\mu$ and $\sigma$ of the predicted distribution respectively. Namely, the prediction process and the training loss are:
\begin{align}
    \mu &= \vect{U}_1\vect{h} + \vect{b}_1\\
    \sigma &= \log(\exp(\vect{U}_2\vect{h} + \vect{b}_2)+1)\\
    p(y|\mu, \sigma) &= \frac{1}{\sqrt{2\pi \sigma^2}}\exp(-\frac{(y-\mu)^2}{2\sigma^2})\\
    \ell(y|\mu, \sigma) &= -\log p(y|\mu, \sigma)
\end{align}
where $\vect{h}$ is the representation vector of the encoder, matrices $\vect{U}_1, \vect{U}_2$, and vectors $\vect{b}_1, \vect{b}_2$ are learnable parameters. $\mu, \sigma$ is the output of the neural networks, and $y$ is the ground truth. In the inference phase, the prediction of the volume is $\mu$. In M-LSTM, hidden states are split into four parts, and the time steps to propagate hidden states of four parts are $1$, $2$, $4$, and $8$, respectively.

The experimental results are in Table~\ref{tab:stoa}. Results show that DeepAR can enhance LSTM models in the volume prediction task, while M-LSTM cannot. ASAT can enhance these methods, which validates its effectiveness. Besides, the Transformer model enhanced with ASAT has a similar performance to the high-performance DeepAR~\citep{DeepAR} method enhanced with ASAT. The Transformer model is a high-performance time-series prediction model. 
\begin{figure}[!t]
\centering
\subcaptionbox{$\mathcal{R}^{(i)}_\text{adv}$ of baseline  model-2. \label{fig:multiple_linear_a}}{\includegraphics[height=1.6 in,width=0.49\linewidth]{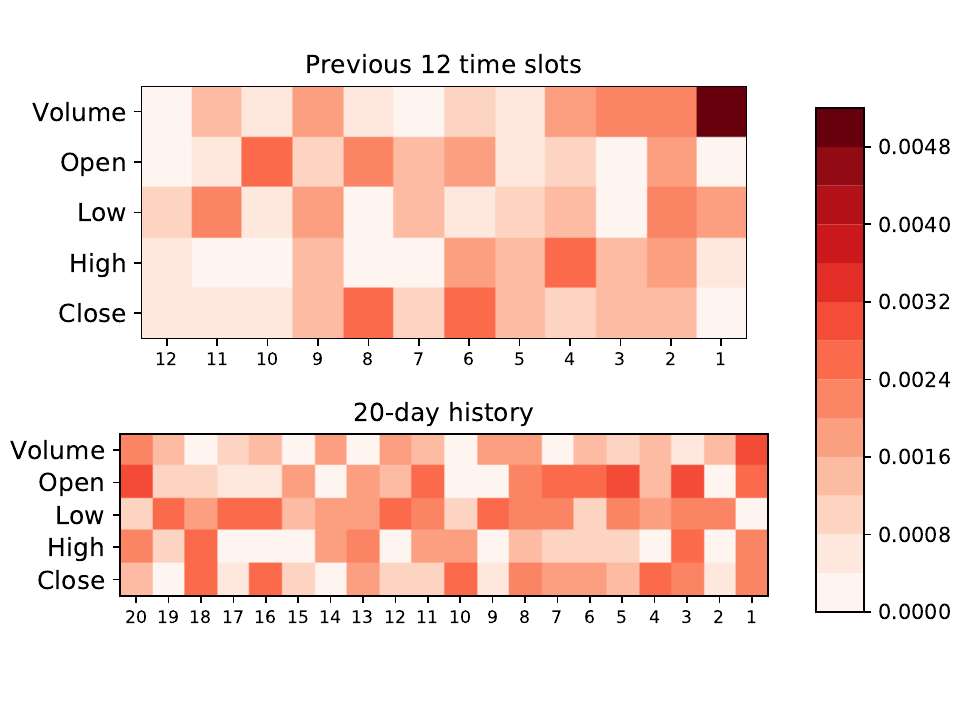}}
\hfil
\subcaptionbox{$\mathcal{R}^{(i)}_\text{adv}$ of baseline  model-3.\label{fig:multiple_linear_b}}{\includegraphics[height=1.6 in,width=0.49\linewidth]{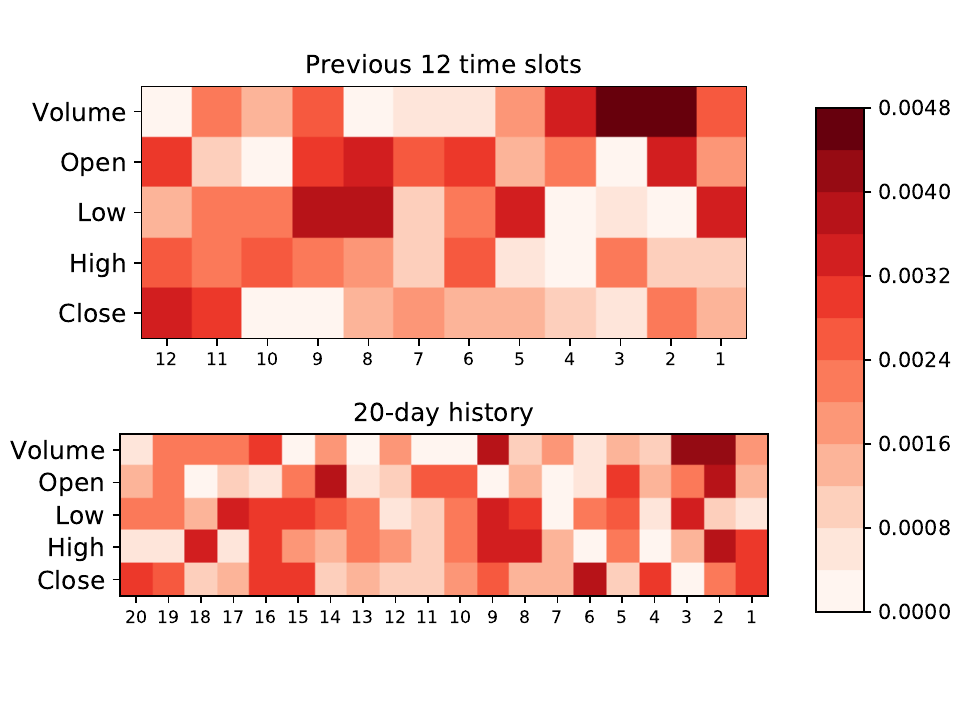}}
\hfil
\subcaptionbox{$\mathcal{R}^{(i)}_\text{adv}$ of ASAT linear model-2.\label{fig:multiple_linear_c}}{\includegraphics[height=1.6 in,width=0.49\linewidth]{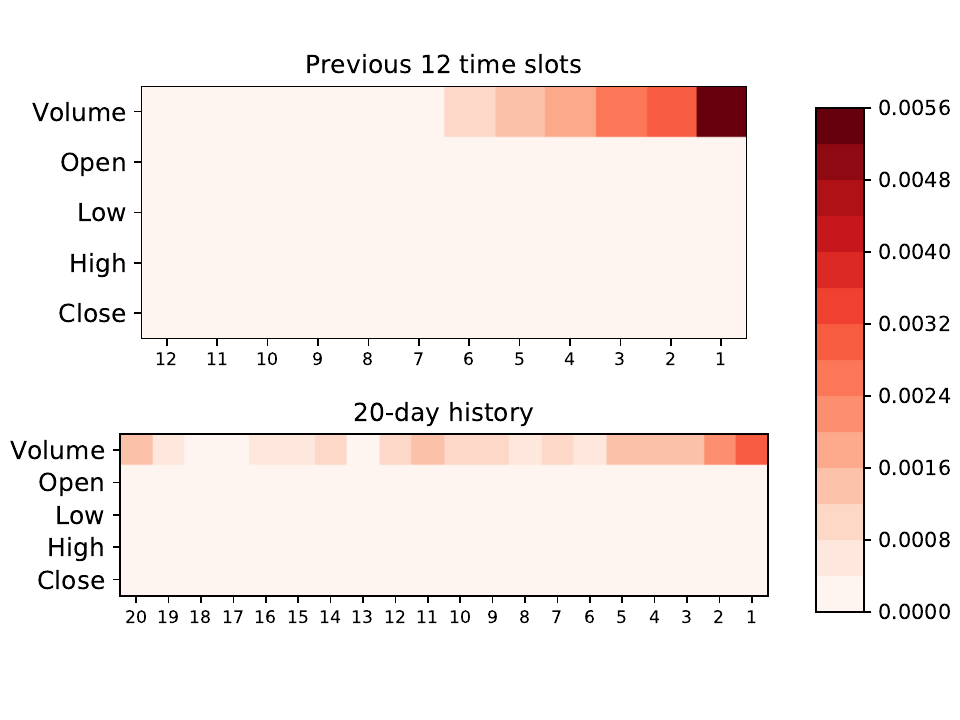}}
\hfil
\subcaptionbox{$\mathcal{R}^{(i)}_\text{adv}$ of ASAT linear model-3.\label{fig:multiple_linear_d}}{\includegraphics[height=1.6 in,width=0.49\linewidth]{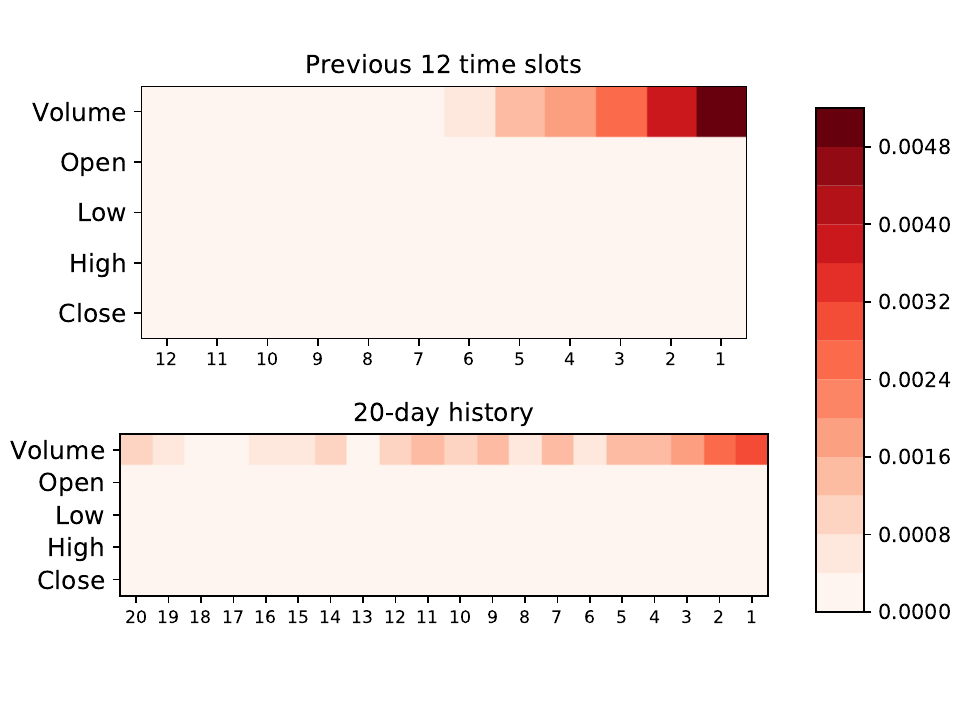}}
\caption{
Visualization of the dimension-wise sensitivities of two baseline linear models and two linear models with ASAT, on the dataset. Baseline models are sensitive to different dimensions during multiple training and tend to overfit some false clues while ASAT can alleviate the overfitting since models with ASAT always focus on similar dimensions.
}
\label{fig:multiple_linear}
\end{figure}

\subsection{Probing Dimension-wise Sensitivities}
\label{sec:overfit}

We probe the dimension-wise sensitivities of different input dimensions in linear models, trained with three random seeds, the visualization results are shown in Fig.~\ref{fig:mlp} and Fig.~\ref{fig:multiple_linear}. Here we choose the five-minute dataset, $\epsilon=1$, and the dimension-wise sensitivities are defined on the whole dataset.

As shown in Fig.~\ref{fig:mlp_a} and Fig.~\ref{fig:mlp_b}, the dimension-wise sensitivities of the linear model are proportional to the absolute values of the weights or the factor loadings in a linear model, which accords with the theoretical results in Proposition~\ref{prop:linear} that the adversarial sensitivity of a single dimension in linear models can reflect the absolute value of the weight or the factor loading. 

Fig.~\ref{fig:mlp_b}, Fig.~\ref{fig:multiple_linear_a} and Fig.~\ref{fig:multiple_linear_b} show the dimension-wise sensitivities of baseline linear models with three random seeds. The dimensions that models pay close attention to are different during multiple training, which indicates that baseline linear models tend to overfit some false clues. However, as shown in Fig.~\ref{fig:mlp_c}, Fig.~\ref{fig:multiple_linear_c} and Fig.~\ref{fig:multiple_linear_d}, models with ASAT tend to pay close attention to some recent time slots and are sensitive to them, which accords with human intuition and is reasonable since recent time slots are more important in time series regularization. It also indicates that ASAT can alleviate overfitting since models with ASAT often focus on similar dimensions during multiple training.

\begin{figure}[!t]
\centering
\subcaptionbox{$\mathcal{R}^{(i)}_\text{adv}$ of ASAT model on dataset.\label{fig:transformer_a}}{\includegraphics[height=1.6 in,width=0.49\linewidth]{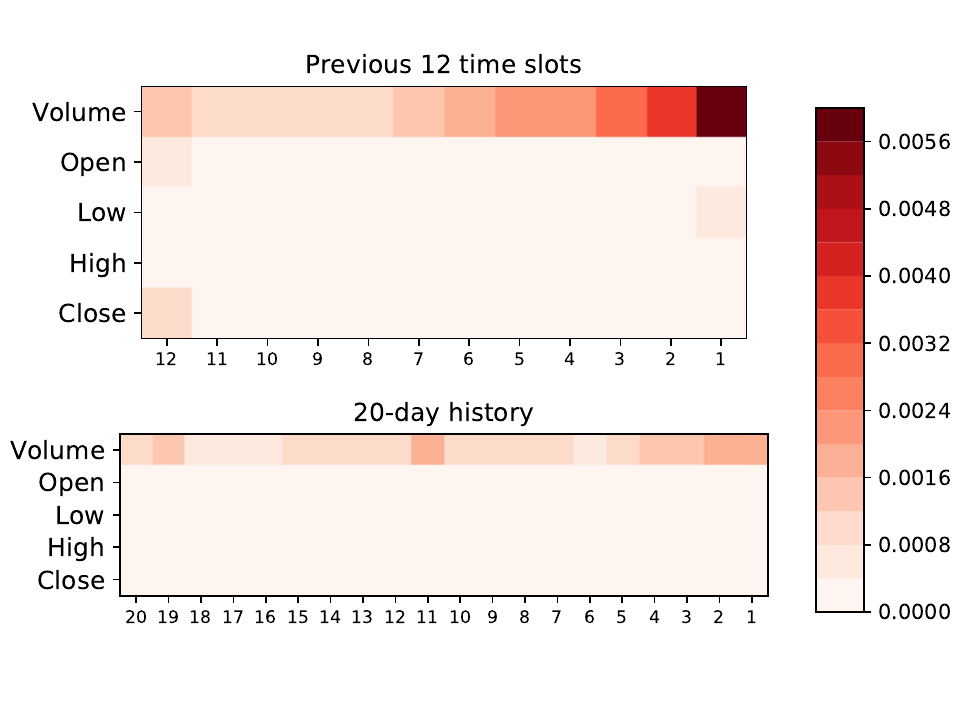}}
\hfil
\subcaptionbox{Volume of input $(\vect{x}, y)$.\label{fig:transformer_b}}{\includegraphics[height=1.6 in,width=0.49\linewidth]{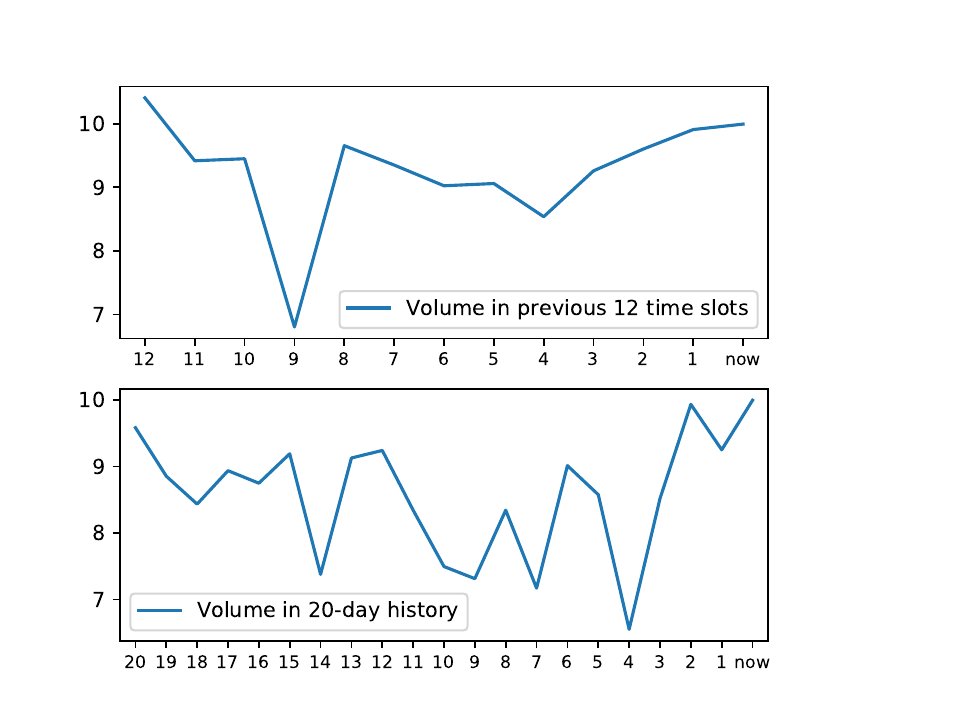}}
\hfil
\subcaptionbox{$\mathcal{R}^{(i)}_\text{adv}$ of baseline on $(\vect{x}, y)$.\label{fig:transformer_c}}{\includegraphics[height=1.6 in,width=0.49\linewidth]{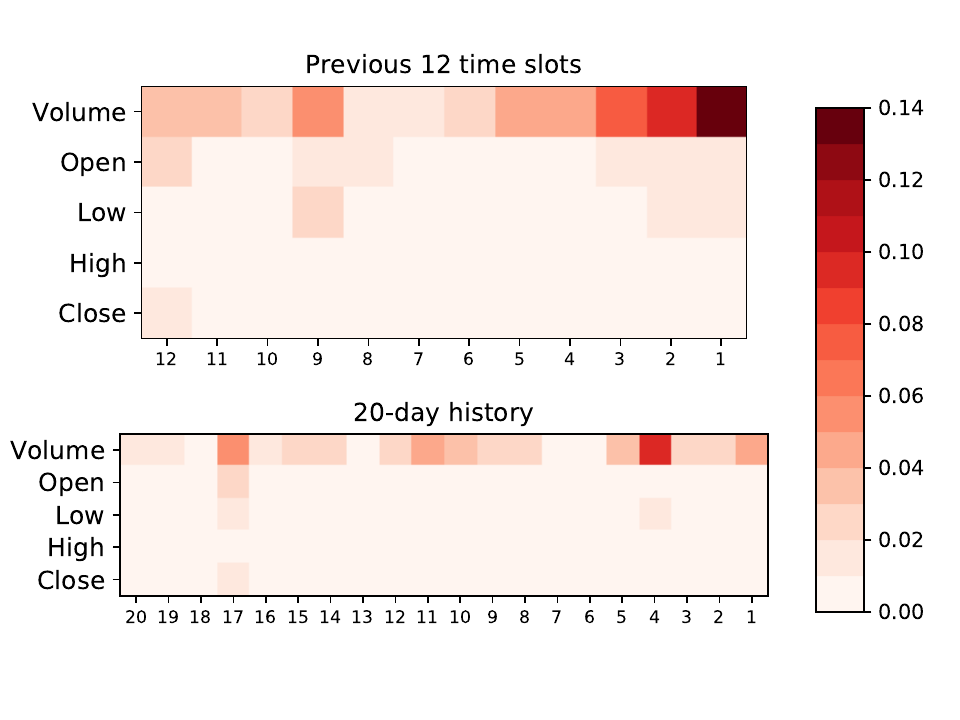}}
\hfil
\subcaptionbox{$\mathcal{R}^{(i)}_\text{adv}$ of ASAT model on $(\vect{x}, y)$.\label{fig:transformer_d}}{\includegraphics[height=1.6 in,width=0.49\linewidth]{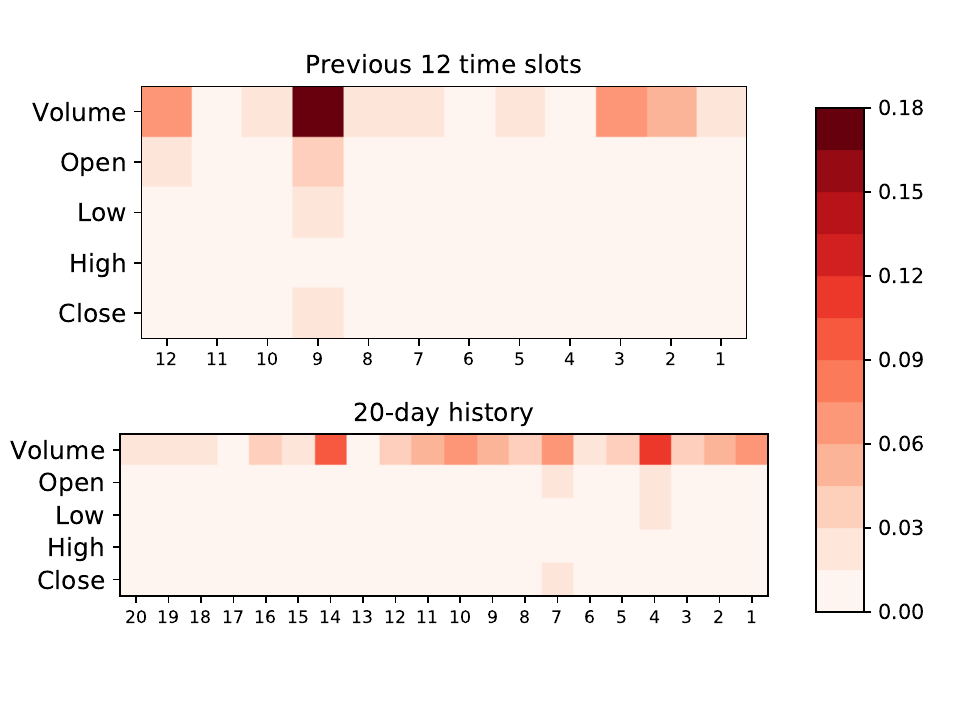}}
\caption{
Visualization of the dimension-wise sensitivities of transformer models, which reveals that models can capture abnormal fluctuations in the data input.
}
\label{fig:transformer}
\end{figure}

\subsection{Analysing Decision Basis on Individual Instance}
\label{sec:single}

In Eq.(\ref{eq:robustness}), the dimension-wise sensitivity can be defined on a dataset $\mathcal{D}$ or a single data instance $(\vect{x}, y)$. When it is defined on a single instance, it can be adopted to probe the decision basis of the model on a single instance. This probing method can provide more interpretability for neural network decisions, and provide insights on explaining the decision basis of a black box model, both on a dataset or an input.

We choose the five-minute dataset, $\epsilon=1$, and randomly choose a data instance in the dataset to plot the dimension-wise sensitivities on the dataset and a single data instance. As shown in Fig.~\ref{fig:transformer}, Fig.~\ref{fig:transformer_a} shows the dimension-wise sensitivities of a Transformer model trained with ASAT on the dataset. We can see that the model mainly concerns the volumes of nearest time slots and volumes in 20-day history. We also probe the decision basis on the data instance, whose volumes are shown in Fig.~\ref{fig:transformer_b}. The volumes of the 9-th nearest time slots, the 14-th nearest day in history, and the 4-th nearest day in history are relatively low compared to other time slots or days in history. Both the baseline Transformer (shown in Fig.~\ref{fig:transformer_c}) and the ASAT Transformer (shown in Fig.~\ref{fig:transformer_d}) are sensitive to these time slots or days. They pay close attention to these abnormal volumes and the corresponding prices. Therefore, their decisions on this data instance can capture these abnormal fluctuations. Moreover, the model with ASAT tends to be more sensitive to abnormal fluctuations, which indicates that ASAT can help the decision bases of models be more reasonable and more explainable.

\section{Related Work}

\subsection{Adversarial Training} 
 
\cite{Intriguing_properties_of_neural_networks} first propose the concept of adversarial examples that can mislead deep neural networks with small malicious perturbations, even in the physical world scenarios~\citep{Adversarial_examples_in_the_physical_world}. Besides generating adversarial examples to attack models~\citep{Intriguing_properties_of_neural_networks,Explaining_and_Harnessing_Adversarial_Examples,Adversarial_examples_in_the_physical_world,Deepfool,Towards_Evaluating_the_Robustness_of_Neural_Networks,textbugger,hotflip,one-pixel}, existing studies also concern improving the adversarial robustness and generalization ability of neural networks via adversarial training~\citep{Explaining_and_Harnessing_Adversarial_Examples,Towards_Evaluating_the_Robustness_of_Neural_Networks,YOPO,ForFree,freeLB}. Adversarial training is widely adopted in both the computer vision field~\citep{Explaining_and_Harnessing_Adversarial_Examples,Towards_Evaluating_the_Robustness_of_Neural_Networks,YOPO,ForFree,stable_training,local_stability} and the natural language processing field~\citep{Robust_Machine_Translation,Robsut_Translation_Doubly_Adversarial,freeLB,robust_domain_text,robust_text_class}. 

\subsection{Generative Adversarial Networks}

Generative Adversarial Networks (GANs)~\citep{DBLP:journals/corr/GAN} are also applications of adversarial training, which train both the generator model and the discriminator model with an adversarial learning target. GANs have been widely used in multiple machine learning fields for generating high-quality synthetic data, such as computer vision~\citep{GAN8, GAN9}, natural language processing~\citep{GAN16, GAN17}, and time series synthesis~\citep{GAN19,GAN21}. In the finance field, several GAN models~\citep{GAN19,GAN21,simonetto2018generatingspikingGAN,zhang2018GANingrids,ramponi2018tcGAN,yoon2019timeGAN,takahashi2019modelingwithGAN,data-Augmentation-GANs,Clinical-time-series-GAN,TTS-GAN} are proposed to generate or synthesize better time-series data that can preserve temporal dynamics in time-series.

\subsection{Time Series Analysis and Volume Prediction}

Trading volume prediction plays an important role in algorithmic trading strategies~\citep{volume_important_nn,VWAP_execution,target_volume_weighted_average_price,Optimal_VWAP_Tracking,Improving_VWAP_strategies,Forecasting_volume_and_volatility,Intra-daily_volume_modeling}. Many efforts are paid to volume prediction~\citep{Daily_volume_forecastigng,Forecasting_intraday_volume_Comparison,trading_volume_in_Chinese,SVM_volume,DCS_volume,kalman_volume,Bayesian_volume,volume-lstm,ensemble_volume}. Machine learning or deep learning methods have many applications in volume prediction. \cite{SVM_volume} propose a dynamic SVM-based approach for volume forecasting and \cite{kalman_volume} adopt a Kalman filter approach. \cite{volume-lstm} first adopt LSTM models in volume prediction. \cite{Bayesian_volume} model the volume forecasting task as Bayesian autoregressive conditional models. \cite{ensemble_volume} propose temporal mixture ensemble models for volume predictions. 

\section{Conclusion}

In this work, we propose adaptively scaled adversarial training (ASAT) to enhance neural networks in time series analysis. The proposed approach adaptively rescales different dimensions of perturbations according to their importance. ASAT can improve both the generalization ability and adversarial robustness of several baseline models on intra-day and inter-day volume prediction tasks compared to the baselines. We also utilize the dimension-wise adversarial sensitivity indicator to explain the decision bases of neural networks and show that ASAT can help the decision bases of models be more reasonable and more explainable.

\section*{Acknowledgements}
This work is supported by a Research Grant from Mizuho Securities Co., Ltd. We sincerely thank Mizuho Securities for valuable domain expert suggestions. The experiment data is provided by Mizuho Securities and Reuters.

\bibliography{main}

\appendix
\onecolumn

\section{Theoretical Details}

\subsection{Proofs of Proposition~\ref{prop:linear}}
\begin{propA}
Consider a linear model $f(\vect{x}, \bm\theta)=\bm\theta^\text{T}\vect{x}+b$. Suppose the loss function $\mathcal{L}(\hat y,y)$ is a function $\ell(\hat y)$ with respect to $\hat y=f(\vect{x}, \bm\theta)$ and $M=\sup|\ell''|$ exists. Then there exists $C>0$, such that:
\begin{align}
\mathcal{R}^{(i)}_\text{adv}(\epsilon)=C\epsilon|\theta_i|+ O\big(\frac{M\theta_i^2\epsilon^2}{2}\big),\quad (i=1,2,\cdots,k),
\end{align}
namely $\mathcal{R}^{(i)}_\text{adv}(\epsilon)\propto|\theta_i|$ holds approximately when $\epsilon$ is small.
\label{propA:linear}
\end{propA}

\begin{proof}
Since when $|\delta_i|\le\epsilon,\delta_{\ne i}=0$, according to Taylor series with Lagrange form of the remainder, there exists $\bm\eta$ near $\vect{x}$, such that:
\begin{align}
\mathcal{L}(f(\vect{x}+\bm\delta, \bm\theta), y)-\mathcal{L}(f(\vect{x}, \bm\theta), y)=\ell'(f(\vect{x}, \bm\theta))(\theta_i\delta_i)+\frac{\ell''(f(\bm\eta, \bm\theta))}{2}(\theta_i\delta_i)^2,
\\
\max\limits_{|\delta_i|\le\epsilon,\delta_{\ne i}=0}\big[\mathcal{L}(f(\vect{x}+\bm\delta, \bm\theta), y)-\mathcal{L}(f(\vect{x}, \bm\theta), y)\big]=|\ell'(f(\vect{x}, \bm\theta))||\theta_i|\epsilon+\frac{\ell''(f(\bm\eta, \bm\theta))}{2}(\theta_i\epsilon)^2,
\end{align}
therefore, when $C=\mathbb{E}_{(\vect{x}, y)\sim\mathcal{D}}|\ell'(f(\vect{x}, \bm\theta))|$, 
\begin{align}
\mathcal{R}^{(i)}_\text{adv}(\epsilon)&=\mathbb{E}_{(\vect{x}, y)\sim\mathcal{D}}\left[\max\limits_{|\delta_i|\le\epsilon,\delta_{\ne i}=0}\big[\mathcal{L}(f(\vect{x}+\bm\delta, \bm\theta), y)-\mathcal{L}(f(\vect{x}, \bm\theta), y)\big]\right]\\
&=C\epsilon|\theta_i|+\frac{\mathbb{E}_{(\vect{x}, y)\sim\mathcal{D}}\ell''(f(\bm\eta, \bm\theta))}{2}(\theta_i\epsilon)^2=C\epsilon|\theta_i|+ O\big(\frac{M\theta_i^2\epsilon^2}{2}\big).
\end{align}
\end{proof}

\subsection{Proofs of Proposition~\ref{prop:risk_non_linear}}

\begin{propA}
Consider a neural network that its loss function is convex and $M$-smooth in the neighborhood of $\bm\theta$ \footnote{Note that $\mathcal{L}$ is only required to be convex and $M$-smooth in the neighborhood of $\bm\theta$ instead of the entire $\mathbb{R}^k$.}. Then, 
\begin{align}
\mathcal{R}^{(i)}_\text{adv}(\epsilon)=\epsilon\big(\mathbb{E}| g_i|\big)+ O\big(\frac{M\epsilon^2}{2}\big),\quad (i=1,2,\cdots,k),
\end{align}
where ${\vect{g}}=\nabla_{\vect{x}}\mathcal{L}( f(\vect{x},\bm\theta), y)$ and its $i$-th dimension is $g_i$. Namely $\mathcal{R}^{(i)}_\text{adv}(\epsilon)\propto\mathbb{E}|g_i|$ holds approximately when $\epsilon$ is small.
\label{propA:risk_non_linear}
\end{propA}

\begin{proof}
Since when $|\delta_i|\le\epsilon,\delta_{\ne i}=0$, according to Taylor series with Lagrange form of the remainder, there exists $\bm\eta$ near $\vect{x}$, such that:
\begin{align}
\mathcal{L}(f(\vect{x}+\bm\delta, \bm\theta), y)-\mathcal{L}(f(\vect{x}, \bm\theta), y)= g_i\delta_i+\frac{\mathcal{L}''_i(f(\bm\eta, \bm\theta), y)}{2}\delta_i^2,
\end{align}
therefore, 
\begin{align}
\mathcal{R}^{(i)}_\text{adv}(\epsilon)&=\mathbb{E}_{(\vect{x}, y)\sim\mathcal{D}}\left[\max\limits_{|\delta_i|\le\epsilon, \delta_{j}=0 (i \ne j)}\big[\mathcal{L}(f(\vect{x}+\bm\delta, \bm\theta), y)-\mathcal{L}(f(\vect{x}, \bm\theta), y)\big]\right]\\
&=\mathbb{E}[\max\limits_{|\delta_i|\le\epsilon, \delta_{j}=0 (i \ne j)} g_i\delta_i]+O(\frac{M}{2}\epsilon^2)=\epsilon\big(\mathbb{E}| g_i|\big)+ O\big(\frac{M\epsilon^2}{2}\big).
\end{align}
\end{proof}

\subsection{Proofs of Theorem~\ref{thm:risk_rescaling}}
\begin{thmA}
Consider a neural network that is convex and $M$-smooth in the neighborhood of $\bm\theta$, suppose $\vect{g}$ denotes the gradient, $\vect{g}=\nabla_{\vect{x}}\mathcal{L}( f(\vect{x},\bm\theta), y)$, then its adversarial sensitivity or risk under our proposed adaptively scaled adversarial attacks where $S=\{\bm\delta:\|\bm\alpha^{-1}\odot\bm\delta\|_p\le\epsilon\}$ is:
\begin{align}
\mathcal{R}_\text{adv}(S)=\big(1+o(1)\big)\epsilon\big(\mathbb{E}\|\bm\alpha\odot{\vect{g}}\|_\frac{p}{p-1}\big).
\end{align}
\label{thmA:risk_rescaling}
\end{thmA}
\begin{proof}
Since $\mathcal{L}$ is convex and $M$-smooth near $\bm\theta$,
\begin{align}
\max\limits_{\bm\delta \in S} \mathcal{L}(f(\vect{x}+\bm\delta, \bm\theta), y)-\mathcal{L}(f(\vect{x}, \bm\theta), y)&=\max\limits_{\bm\delta \in S}\vect{g}\cdot\bm\delta+O(\frac{M}{2}\|\bm\delta\|_2^2)\\
&= \max\limits_{\bm\delta \in S}\vect{g}\cdot\bm\delta +o(\epsilon).
\end{align}
According to Holder Inequality, 
\begin{align}
    \vect{g}\cdot\bm\delta=(\vect{g}\odot\bm\alpha)\cdot(\bm\delta\odot\bm\alpha^{-1})\le \|\vect{g}\odot\bm\alpha\|_q\|\bm\delta\odot\bm\alpha^{-1}\|_p=\epsilon\|\vect{g}\odot\bm\alpha\|_q,
\end{align}
where $\frac{1}{p}+\frac{1}{q}=1$, $q=\frac{p}{p-1}$.

Therefore, 
\begin{align}
\mathcal{R}_\text{adv}(S)&=\mathbb{E}_{(\vect{x}, y)\sim\mathcal{D}}\left[\max\limits_{\bm\delta \in S} \mathcal{L}(f(\vect{x}+\bm\delta, \bm\theta), y)-\mathcal{L}(f(\vect{x}, \bm\theta), y)\right]\\
&=\mathbb{E}_{(\vect{x}, y)\sim\mathcal{D}}\left[\max\limits_{\bm\delta \in S}\vect{g}\cdot\bm\delta\right] +o(\epsilon)\\
&=\mathbb{E}_{(\vect{x}, y)\sim\mathcal{D}}\left[\epsilon\|\vect{g}\odot\bm\alpha\|_q\right] +o(\epsilon)\\
&=\big(1+o(1)\big)\epsilon\big(\mathbb{E}\|\bm\alpha\odot{\vect{g}}\|_\frac{p}{p-1}\big).
\end{align}

\end{proof}

\subsection{Proofs of Theorem~\ref{thm:generalization_error}}

\begin{thmA}
\label{thmA:generalization_error}
Suppose the training distribution $\mathcal{P}:=p(\vect{z})$ and the test distribution $\mathcal{Q}:=q(\vect{z})$ are two close distributions that have definite first and second order distributional shift distances, namely $D_1<+\infty, D_2<+\infty$, for any loss function $\mathcal{L}$ that maps an instance to $[0, 1]$ and is convex and $M$-smooth in the neighborhood of $\bm\theta$, with probability 1-$\omega$ over the choice of the training set $\mathcal{D}\sim \mathcal{P}$, the following bound holds:
\begin{align}
\mathcal{L}_\mathcal{Q}\le
\mathcal{L}_\mathcal{D}+\frac{\big(1+o(1)\big)k^\frac{1}{p}\mathcal{R}_\text{adv}(S)}{\epsilon}+\frac{MD_2}{2}+\text{Rem},
\end{align}
where $S=\{\bm\delta:\|\bm\beta^{-1}\odot\bm\delta\|_p\le\epsilon\}$, and the remainder term is $\text{Rem}=\sqrt{\frac{\log\frac{|\mathcal{D}|}{\omega}+2\log|\mathcal{D}|}{2|\mathcal{D}|}}+\frac{1}{|\mathcal{D}|}$. 
\end{thmA}
\begin{proof}
The main contribution of our work is Lemma~\ref{lemma:risk},
\begin{lem}
\label{lemma:risk}
For $S=\{\bm\delta:\|\bm\beta^{-1}\odot\bm\delta\|_p\le\epsilon\}$,
\begin{align}
\mathcal{L}_\mathcal{Q}\le
\mathcal{L}_\mathcal{P}+\frac{\big(1+o(1)\big)k^\frac{1}{p}\mathcal{R}_\text{adv}(S)}{\epsilon}+\frac{MD_2}{2}.
\end{align}
\end{lem}

We prove Lemma~\ref{lemma:risk} first. Suppose $\vect{x}, y = \phi(\vect{z}')$ where $\vect{z}'=(\vect{x}', y)$, note that:
\begin{align}
    \mathcal{L}_\mathcal{Q}-\mathcal{L}_\mathcal{P}&=\mathbb{E}_{\vect{z}\sim q(\vect{z})}(\mathcal{L})-\mathbb{E}_{\vect{z}\sim p(\vect{z})}(\mathcal{L})\\
    &=\int q(\vect{z})\mathcal{L}d\vect{z}-\int p(\vect{z})\mathcal{L}d\vect{z}\\
    &=\int q(\phi(\vect{z}'))\mathcal{L}(f(\vect{x}, \bm\theta), y)d\phi(\vect{z}')-\int p(\vect{z})\mathcal{L}(f(\vect{x}, \bm\theta), y)d\vect{z}\\    
    &=\int p(\vect{z}')\mathcal{L}(f(\vect{x}, \bm\theta), y)d\vect{z}'-\int p(\vect{z})\mathcal{L}(f(\vect{x}, \bm\theta), y)d\vect{z}\\
    &=\int p(\vect{z}')\big[\mathcal{L}(f(\vect{x}, \bm\theta), y)-\mathcal{L}(f(\vect{x}', \bm\theta), y)\big]d\vect{z}'\\
    &\le \int p(\vect{z}')\big[\vect{g}\cdot (\vect{z}-\vect{z}')+\frac{M}{2}\|\vect{z}-\vect{z}'\|_2^2\big]d\vect{z}'\\
    &=\mathbb{E}_{\vect{z}\sim p(\vect{z})}\big[\vect{g}\cdot (\phi(\vect{z})-\vect{z})+\frac{M}{2}\|\phi(\vect{z})-\vect{z}\|_2^2\big].
\end{align}

Note that for $r=\frac{p}{p-1}\ge 1$, $\|\vect{v}\|_r=\big(\sum\limits_{i=1}^k|v_i|^{r}\big)^{\frac{1}{r}}=k^\frac{1}{r} \big(\frac{\sum\limits_{i=1}^k|v_i|^r}{k}\big)^{\frac{1}{r}} \ge k^\frac{1}{r} \big(\frac{\sum\limits_{i=1}^k|v_i|}{k}\big)=  k^{\frac{1}{r}-1}(\sum\limits_{i=1}^k|v_i|) = k^{\frac{1}{r}-1}\|\vect{v}\|_1=k^{-\frac{1}{p}}\|\vect{v}\|_1$, therefore,
\begin{align}
    \mathcal{L}_\mathcal{Q}&\le\mathcal{L}_\mathcal{P}+\mathbb{E}_{\vect{z}\sim p(\vect{z})}[{\vect{g}}\cdot(\phi(\vect{z})-\vect{z})]+\frac{M}{2}\mathbb{E}_{\vect{z}\sim p(\vect{z})}\|\phi(\vect{z})-\vect{z}\|_2^2\\
    &= \mathcal{L}_\mathcal{P}+\mathbb{E}_{\vect{z}\sim p(\vect{z})}[{\vect{g}}\cdot\bm\beta]+\frac{MD_2}{2} \\    
    &\le \mathcal{L}_\mathcal{P}+\mathbb{E}_{\vect{z}\sim p(\vect{z})}\|{\vect{g}}\odot\bm\beta\|_1+\frac{MD_2}{2} \\
    &\le \mathcal{L}_\mathcal{P}+k^\frac{1}{p}\mathbb{E}_{\vect{z}\sim p(\vect{z})}\|{\vect{g}}\odot\bm\beta\|_\frac{p}{p-1}+\frac{MD_2}{2}\\
    &=\mathcal{L}_\mathcal{P}+\frac{\big(1+o(1)\big)k^\frac{1}{p}\mathcal{R}_\text{adv}(S)}{\epsilon}+\frac{MD_2}{2},
\end{align}
where $S=\{\bm\delta:\|\bm\beta^{-1}\odot\bm\delta\|_p\le\epsilon\}$, and $\mathcal{R}_\text{adv}(S)=\big(1+o(1)\big)\epsilon\big(\mathbb{E}_{\vect{z}\sim p(\vect{z})}\|\bm\beta\odot{\vect{g}}\|_\frac{p}{p-1}\big)$.

Combined with Lemma~\ref{lemma:Bayes} in \cite{Some_PAC-Bayesian_Theorems}, the theorem is proven.
\begin{lem}
\label{lemma:Bayes}
For any loss function $\mathcal{L}$ that maps an instance to $[0, 1]$, with probability 1-$\omega$ over the choice of the training, set $\mathcal{D}\sim \mathcal{P}$, the following bound holds:
\begin{align}
\mathcal{L}_\mathcal{P}\le
\mathcal{L}_\mathcal{D}+\sqrt{\frac{\log\frac{|\mathcal{D}|}{\omega}+2\log|\mathcal{D}|}{2|\mathcal{D}|}}+\frac{1}{|\mathcal{D}|},
\end{align}
\end{lem}

\end{proof}

\subsection{Closed-form Solutions in FSM and PGD Algorithms after Rescaling}

After rescaling, the constraint set is $S=\{\bm\delta:\|\bm\alpha^{-1}\odot\bm\delta\|_p\le\epsilon\}$. For the fast gradient method target, namely maximizing the inner product of $\bm \delta$ and the gradient $\vect{g}$, the solution is shown in Corollary~\ref{cor:fgsm}.

\begin{cor}
Suppose the constraint set is $S=\{\bm\delta:\|\bm\alpha^{-1}\odot\bm\delta\|_p\le\epsilon\}$, the following are the solution of the optimization problem:
\begin{align}
\bm\delta = \argmax_{\bm\delta\in S}\bm\delta^{\text{T}}\vect{g}=\epsilon\big(\bm\alpha\odot\text{sgn}(\vect{g})\big)\odot\frac{|\bm\alpha\odot\vect{g}|^\frac{1}{p-1}}{\||\bm\alpha\odot\vect{g}|^\frac{1}{p-1}\|_p}.
\end{align}
\label{cor:fgsm}
\end{cor}

\begin{proof}
Suppose $\frac{1}{p}+\frac{1}{q}=1$, according to the Holder Inequality,
\begin{align}
\bm\delta^{\text{T}}\vect{g}=
(\bm\alpha^{-1}\odot\bm\delta)^{\text{T}}(\bm\alpha\odot\vect{g})\le\|\bm\alpha^{-1}\odot\bm\delta\|_p\|\bm\alpha\odot\vect{g}\|_q=\epsilon\|\bm\alpha\odot\vect{g}\|_q.
\end{align}

The equation holds if and only if, there exists $\lambda>0$, such that:
\begin{align}
|\bm\alpha^{-1}\odot\bm\delta|^p = \lambda|\bm\alpha\odot\vect{g}|^q ,\quad \|\bm\alpha^{-1}\odot\bm\delta\|_p =\epsilon \text{ , and}\quad \text{sgn}(\bm\alpha^{-1}\odot\bm\delta)=\text{sgn}(\bm\alpha\odot\vect{g}),
\end{align}
that is to say,
\begin{align}
\bm\alpha^{-1}\odot\bm\delta = \lambda \text{sgn}(\bm\alpha\odot\vect{g})\odot|\bm\alpha\odot\vect{g}|^{\frac{q}{p}} = \lambda \text{sgn}(\vect{g})\odot|\bm\alpha\odot\vect{g}|^{\frac{1}{p-1}}.
\end{align}

According to $\|\bm\alpha^{-1}\odot\bm\delta\|_p=\epsilon$, we can solve the $\lambda$,
\begin{align}
\|\bm\alpha^{-1}\odot\bm\delta\| = \|\text{sgn}(\vect{g})\odot\lambda|\bm\alpha\odot\vect{g}|^{\frac{1}{p-1}}\|_p=\lambda \||\bm\alpha\odot\vect{g}|^{\frac{1}{p-1}}\|_p=\epsilon, \quad \lambda = \frac{\epsilon}{\||\bm\alpha\odot\vect{g}|^{\frac{1}{p-1}}\|_p}.
\end{align}

To conclude,
\begin{align}
\bm\alpha^{-1}\odot\bm\delta &= \epsilon\text{sgn}(\vect{g})\odot\frac{|\bm\alpha\odot\vect{g}|^\frac{1}{p-1}}{\||\bm\alpha\odot\vect{g}|^\frac{1}{p-1}\|_p},\\
\bm\delta &=\epsilon\big(\bm\alpha\odot\text{sgn}(\vect{g})\big)\odot\frac{|\bm\alpha\odot\vect{g}|^\frac{1}{p-1}}{\||\bm\alpha\odot\vect{g}|^\frac{1}{p-1}\|_p}.
\end{align}
\end{proof}

For the projection function $\Pi_S(\vect{v})$ that projects $\vect{v}$ into the set $S$,  we adopt the following variants:
\begin{align}
\Pi_{\{\vect{v}:\|\bm\alpha^{-1}\odot\vect{v}\|_2\le\epsilon\}}{(\vect{v})} =& \min\{\|\bm\alpha^{-1}\odot\vect{v}\|_2,\epsilon\}\frac{\vect{v}}{\|\bm\alpha^{-1}\odot\vect{v}\|_2} \label{eqA:projection_new_L2},\\
\Pi_{\{\vect{v}:\|\bm\alpha^{-1}\odot\vect{v}\|_{+\infty}\le\epsilon\}}{(\vect{v})} &= \bm\alpha\odot\text{clip}(\bm\alpha^{-1}\odot\vect{v},-\epsilon, \epsilon) \label{eqA:projection_new_Linf}.
\end{align}

The constraint set is $S=\{\bm\delta:\|\bm\alpha^{-1}\odot\bm\delta\|_p\le\epsilon\}$. Define $\vect{u}=\bm\alpha^{-1}\odot\vect{v}, S'=\{\vect{u}:\|\vect{u}\|_p\le\epsilon\}$. Consider a mapping $\phi(\vect{v})=\alpha\odot\vect{v}$, then $\phi(\vect{u})=\vect{v}$ and $\phi(S')=S$.

Since $\vect{u}$ is projected to $\Pi_{S'}(\vect{u})$ in $S'$, we adopt the variant that projects $\vect{v}=\phi(\vect{u})$ to $\phi(\Pi_{S'}(\vect{u}))$, namely,
\begin{align}
\Pi_{S}(\vect{v})=\phi(\Pi_{S'}(\vect{u}))=\bm\alpha\odot(\Pi_{S'}(\vect{u}))=\bm\alpha\odot(\Pi_{S'}(\bm\alpha^{-1}\odot\vect{v})),
\end{align}
therefore,
\begin{align}
\Pi_{\{\vect{v}:\|\bm\alpha^{-1}\odot\vect{v}\|_2\le\epsilon\}}{(\vect{v})} =& \min\{\|\vect{u}\|_2,\epsilon\}\frac{\bm\alpha\odot\vect{u}}{\|\vect{u}\|_2} = \min\{\|\bm\alpha^{-1}\odot\vect{v}\|_2,\epsilon\}\frac{\vect{v}}{\|\bm\alpha^{-1}\odot\vect{v}\|_2},\\
\Pi_{\{\vect{v}:\|\bm\alpha^{-1}\odot\vect{v}\|_{+\infty}\le\epsilon\}}{(\vect{v})} &= \bm\alpha\odot\text{clip}(\vect{u},-\epsilon, \epsilon) =  \bm\alpha\odot\text{clip}(\bm\alpha^{-1}\odot\vect{v},-\epsilon, \epsilon).
\end{align}

\section{Hypothesis Tests}

We also conduct student-$t$ tests on the hourly dataset to verify that the proposed ASAT outperforms baselines statistically significantly $(p<0.05)$. The null hypothesis is that the outperform of ASAT is worse than the baseline.  The degree of freedom is df$=5+5-2=8$ and the critical $t$-value is $1.86$. Experimental results of hypothesis tests are in Table~\ref{tab:t-tests}. We can see that $t$-values of MLP and Transformer models are larger than $1.86$. For the LSTM model, the improvements in ASAT are not significant.

\begin{table*}[!h]
\caption{Experimental results of hypothesis tests. AT denotes adversarial training, and ASAT denotes adaptively scaled adversarial training.}
\scriptsize
\label{tab:t-tests}
\setlength{\tabcolsep}{4pt}
\centering
\begin{tabular}{@{}lcccc@{}}
\toprule
 \bf Dataset &  \multicolumn{4}{c}{\bf Hourly}\\ 
 \midrule
 \bf Model & \bf MSE & \bf RMSE  & \bf MAE  & \bf ACC\\ 
 \midrule
 20-day average & {0.700} & {0.837} & {0.608} & {0.709} \\
 \midrule
 Linear & 0.227$\pm$0.019 & 0.477$\pm$0.019 & 0.370$\pm$0.023 & 0.708$\pm$0.012 \\
 +ASAT &  {0.206$\pm$0.004} & {0.454$\pm$0.004} & {0.347$\pm$0.006} & 0.720$\pm$0.007 \\
 $t$-values & 2.41 & 2.65 & 2.16 & 1.93\\
 \midrule
 LSTM & 0.223$\pm$0.005 & 0.472$\pm$0.006 & 0.361$\pm$0.007 & 0.708$\pm$0.004\\
 +ASAT &  {0.220$\pm$0.002} & {0.469$\pm$0.003} & {0.359$\pm$0.005} & {0.712$\pm$0.003}  \\
  $t$-values & 1.25 & 1.00 & 0.52 & 1.79\\
 \midrule
 Transformer & 0.219$\pm$0.012 & 0.470$\pm$0.013 & 0.357$\pm$0.012 & 0.711$\pm$0.012 \\
 +ASAT &  {0.201$\pm$0.002} & {0.449$\pm$0.002} & {0.342$\pm$0.004} & {0.726$\pm$0.003} \\
  $t$-values & 3.31 & 3.57 & 2.65 & 2.71\\
\bottomrule
\end{tabular}
\end{table*}

\section{Details of Hyperparameters Choice}

We try $L_2$ and $L_{+\infty}$ constraints, try $K$ in $\{1, 2, 3\}$, and search the hyperparameter $\epsilon$ in $\{$0.001, 0.002, 0.005, 0.01, 0.02, 0.05, 0.1, 0.2, 0.5, 1$\}$. For time-dependently scaled adversarial training, we try $\gamma$ in $\{0.01, 0.02, 0.03\}$ in the Linear decay, and try $\gamma$ in $\{0.7, 0.8, 0.9\}$ in the Exp decay. To investigate the influence of hyperparameters $\epsilon$ and $\gamma$, we take the Transformer model on the five-minute dataset as an example and discuss the influence of $\epsilon$ according to the experimental results. Then, we report the best configurations of hyperparameters for every task. 

\subsection{The Influence of Hyperparameter}

We conduct experiments to investigate the influence of the hyperparameter $\epsilon$. Experimental results are shown in Table~\ref{tab:epsilon}, where the $L_{+\infty}$ constraint is adopted. The process of hyperparameter search shows that too large or small $\epsilon$ cannot improve the model accuracy well, and an appropriate $\epsilon$ needs to be selected. 

\begin{table}[H]
\caption{Experimental results of different $\epsilon$.}
\scriptsize
\label{tab:epsilon}
\setlength{\tabcolsep}{5pt}
\centering
\begin{tabular}{@{}lcccc@{}}
\toprule
 \bf Settings & \bf MSE & \bf RMSE  & \bf MAE  & \bf ACC \\ 
 \midrule
 20-day average & {0.700} & {0.837} & {0.608} & {0.709}\\
 \midrule
 Transformer & 0.665$\pm$0.044 & {0.815}$\pm$0.027 & 0.610$\pm$0.032 & 0.695$\pm$0.019  \\
 \midrule
 +ASAT ($K=1, \epsilon=0.001$) & 0.672$\pm$0.066 & 0.819$\pm$0.039 & 0.616$\pm$0.050&0.692$\pm$0.030 \\
 +ASAT ($K=1, \epsilon=0.002$) & \textbf{0.625$\pm$0.023}& \textbf{0.791$\pm$0.015} &	\textbf{0.580$\pm$0.017}&\textbf{0.712$\pm$0.010} \\
 +ASAT ($K=1, \epsilon=0.005$) & 0.635$\pm$0.025& 0.797$\pm$0.016 &	0.588$\pm$0.020&0.708$\pm$0.011 \\
  \midrule 
  +ASAT ($K=2, \epsilon=0.005$) & 0.654$\pm$0.073 & 0.807$\pm$0.043 & 0.603$\pm$0.054&0.697$\pm$0.031 \\
 +ASAT ($K=2, \epsilon=0.01$) & \textbf{0.630$\pm$0.027}& \textbf{0.794$\pm$0.017} &	\textbf{0.584$\pm$0.020}&\textbf{0.708$\pm$0.011} \\
 +ASAT ($K=2, \epsilon=0.02$) & 0.665$\pm$0.068& 0.814$\pm$0.041 &	0.608$\pm$0.049&0.695$\pm$0.027 \\
  \midrule
  +ASAT ($K=3, \epsilon=0.005$) & 0.650$\pm$0.047 & 0.806$\pm$0.028 & 0.601$\pm$0.035&0.697$\pm$0.023 \\
 +ASAT ($K=3, \epsilon=0.01$) & \textbf{0.645$\pm$0.039}& \textbf{0.802$\pm$0.024} &	\textbf{0.597$\pm$0.029}&\textbf{0.701$\pm$0.018} \\
 +ASAT ($K=3, \epsilon=0.02$) & 0.651$\pm$0.040& 0.807$\pm$0.016 &	0.602$\pm$0.030&0.697$\pm$0.017 \\
\bottomrule
\end{tabular}
\end{table}

\subsection{Detailed Hyperparameters Settings}

We conduct hyperparameter search experiments to find the best configurations of hyperparameters. 

\subsubsection{Settings on the Linear Model}

On the hourly dataset, for traditional adversarial training, and time-dependently scaled adversarial training, we adopt the $L_2$ constraint, $K=1$, and $\epsilon=0.5$. For time-dependently scaled adversarial training, $\gamma=0.01$ in the Linear decay, and $\gamma=0.9$ in the Exp decay. For ASAT and other variants, we adopt the $L_{+\infty}$ constraint, $K=3$, and $\epsilon=0.2$.

On the five-minute dataset, for traditional adversarial training, and time-dependently scaled adversarial training, we adopt the $L_2$ constraint, $K=1$, and $\epsilon=0.001$. For time-dependently scaled adversarial training, $\gamma=0.02$ in the Linear decay, and $\gamma=0.9$ in the Exp decay. For ASAT and other variants, we adopt the $L_{+\infty}$ constraint, $K=3$, and $\epsilon=0.5$.

\subsubsection{Settings on the LSTM Model}

On the hourly dataset, for traditional adversarial training, and time-dependently scaled adversarial training, we adopt the $L_2$ constraint, $K=1$, and $\epsilon=0.001$. For time-dependently scaled adversarial training, $\gamma=0.01$ in the Linear decay, and $\gamma=0.7$ in the Exp decay. For ASAT and other variants, we adopt the $L_{2}$ constraint, $K=3$, and $\epsilon=0.1$.

On the five-minute dataset, for traditional adversarial training, and time-dependently scaled adversarial training, we adopt the $L_2$ constraint, $K=1$, and $\epsilon=0.001$. For time-dependently scaled adversarial training, $\gamma=0.01$ in the Linear decay, and $\gamma=0.7$ in the Exp decay. For ASAT and other variants, we adopt the $L_{2}$ constraint, $K=3$, and $\epsilon=0.2$.

\subsubsection{Settings on the Transformer Model}

On the hourly dataset, for traditional adversarial training, and time-dependently scaled adversarial training, we adopt the $L_2$ constraint, $K=2$, and $\epsilon=0.001$. For time-dependently scaled adversarial training, $\gamma=0.01$ in the Linear decay, and $\gamma=0.9$ in the Exp decay. For ASAT and other variants, we adopt the $L_{+\infty}$ constraint, $K=1$, and $\epsilon=0.002$.

On the five-minute dataset, for traditional adversarial training, and time-dependently scaled adversarial training, we adopt the $L_{+\infty}$ constraint, $K=2$, and $\epsilon=0.02$. For time-dependently scaled adversarial training, $\gamma=0.01$ in the Linear decay, and $\gamma=0.9$ in the Exp decay. For ASAT and other variants, we adopt the $L_{+\infty}$ constraint, $K=3$, and $\epsilon=0.001$.

\subsubsection{Settings for Enhancing Existing Methods.}

We train every model for $5$ epochs. We repeat every experiment with $5$ runs. In DeepAR, we adopt the $L_{+\infty}$ constraint, $K=2$, and $\epsilon=0.001$. In M-LSTM, we choose the time scales as $\{1, 2, 4, 8\}$. We adopt the $L_2$ constraint, $K=2$, and $\epsilon=0.2$.

\end{document}